\documentclass[11pt, reqno]{amsart}
\usepackage{prettyref}
\usepackage{amstext}
\usepackage{amsthm}
\usepackage{amssymb}
\usepackage{enumerate,caption,subcaption, graphicx}

\usepackage{prettyref}
\usepackage{tikz}

\usepackage{xcolor}
\usepackage[colorlinks=true]{hyperref}

\makeatletter
\numberwithin{equation}{section}
\numberwithin{figure}{section}
\theoremstyle{plain}
\newtheorem{thm}{\protect\theoremname}[section]
\theoremstyle{definition}
\newtheorem{defn}[thm]{\protect\definitionname}
\theoremstyle{plain}
\newtheorem{prop}[thm]{\protect\propositionname}
\theoremstyle{plain}
\newtheorem{cor}[thm]{\protect\corollaryname}
\theoremstyle{remark}
\newtheorem{rem}[thm]{\protect\remarkname}
\theoremstyle{plain}
\newtheorem{lem}[thm]{\protect\lemmaname}
\theoremstyle{definition}

\theoremstyle{plain}
\newtheorem*{thm*}{\protect\theoremname}
\newtheorem*{lem*}{\protect\lemmaname}

\usepackage{mathrsfs,amsthm,amssymb,bbm,fullpage}

\renewcommand{\liminf}{\varliminf}
\renewcommand{\limsup}{\varlimsup}

\newcommand{\cN}{\mathcal{N}}

\newcommand{\cF}{\mathcal{F}}

\newcommand{\cL}{\mathcal{L}}

\newcommand{\cP}{\mathcal{P}}

\newcommand{\bS}{\mathbb{S}}
\newcommand{\R}{\mathbb{R}}

\newcommand{\prob}{\mathbb{P}}
\newcommand{\E}{\mathbb{E}}

\newcommand{\eps}{\epsilon}

\newcommand{\indicator}[1]{\mathbf{1}_{#1}}

\newcommand{\eqdist}{\stackrel{(d)}{=}}
\newcommand{\tensor}{\otimes}

\newcommand{\abs}[1]{\lvert#1\rvert}
\newcommand{\norm}[1]{\lvert\lvert#1\rvert\rvert}

\newrefformat{prop}{Proposition \ref{#1}}
\newrefformat{cor}{Corollary \ref{#1}}
\newrefformat{subsec}{Section \ref{#1}}
\newrefformat{app}{Appendix \ref{#1}}

\makeatother

\providecommand{\corollaryname}{Corollary}
\providecommand{\definitionname}{Definition}
\providecommand{\examplename}{Example}
\providecommand{\lemmaname}{Lemma}
\providecommand{\propositionname}{Proposition}
\providecommand{\remarkname}{Remark}
\providecommand{\theoremname}{Theorem}

\numberwithin{equation}{section}

\begin{document}
\title{
Online stochastic gradient descent on non-convex losses from high-dimensional inference
}

\author{Gerard Ben Arous}
\address{G.\ Ben Arous\hfill\break
Courant Institute \\ New York University\\
New York, NY, USA.}
\email{benarous@cims.nyu.edu}

\author{Reza Gheissari}
\address{R.\ Gheissari\hfill\break
Departments of Statistics and EECS \\ University of California, Berkeley\\
Berkeley, CA, USA.}
\email{gheissari@berkeley.edu}

\author{Aukosh Jagannath}
\address{A.\ Jagannath\hfill\break
Departments of Statistics and Actuarial Science and Applied Mathematics \\ University of Waterloo\\
Waterloo, ON, Canada.} 
\email{a.jagannath@uwaterloo.ca}

\begin{abstract}
Stochastic gradient descent (SGD) is a popular algorithm for optimization problems arising in high-dimensional inference tasks. Here one produces an estimator of an unknown parameter from independent samples of data by iteratively optimizing a loss function. 
This loss function is random and often non-convex. We study the performance of the simplest version of SGD, namely online SGD, from a random start in the setting where the parameter space is high-dimensional. 

 We develop nearly sharp thresholds for the number of samples needed for consistent estimation as one varies the dimension. Our thresholds depend only on an intrinsic property of the population loss which we call the information exponent. In particular, our results do not assume uniform control on the loss itself, such as convexity or uniform derivative bounds. The thresholds we obtain are polynomial in the dimension and the precise exponent depends explicitly on the information exponent. As a consequence of our results, we find that except for the simplest tasks, almost all of the data is used simply in the initial search phase to obtain non-trivial correlation with the ground truth.  Upon attaining non-trivial correlation, the descent is rapid and exhibits law of large numbers type behavior. 
 
  We illustrate our approach by applying it to a wide set of inference tasks such as phase retrieval, and  parameter estimation for generalized linear models, online PCA, and spiked tensor models, as well as to supervised learning for single-layer networks with general activation functions.
\end{abstract}

\maketitle

\vspace{-.75cm}
 \section{Introduction}

Stochastic gradient descent (SGD)  and its many variants are the algorithms of choice for many hard optimization problems encountered in machine learning and data science.
Since its introduction in the 1950s \cite{RobMon51}, SGD has been abundantly studied. While first analyzed in fixed dimensions \cite{mcleish,BMP90,BKM95,duflo96, Bottou99,Benaim99}, its analysis in high-dimensional settings has recently become the subject of intense interest, both from a theoretical point of view (see, e.g.,~ \cite{NeedellSrebroWard14,escaping-saddles-COLT,online-ICA-NIPS,WaLu16,WaMaLu16,mandt2017stochastic,HLPR-COLT,TanVershynin}) and a practical one (see, e.g.,~\cite{lecun1998gradient,bishop,goodfellow2016deep}).

The evolution of SGD is often heuristically viewed as having two phases \cite{bottou2003stochastic,BottouLeCun04,mandt2017stochastic}: an initial ``search'' phase and a final ``descent'' phase. In the search phase, one often thinks of the algorithm as wandering in a non-convex landscape. In the descent phase, however, one views the algorithm as being in an effective trust region and descending quickly to a local minimum. 
In this latter phase, one expects the algorithm to be a good approximation to the gradient descent for the population loss, whereas in the former the quality of this approximation should be quite low.  

In the fixed dimension setting one can avoid an analysis of the search phase by neglecting an initial burn-in period, and assuming that the algorithm starts in the descent phase. This reasoning is sometimes used \cite{Bottou99} as a motivation for convexity or quasi-convexity assumptions in the analysis of stochastic gradient descent for which there now is a large literature \cite{Bottou99,BottouLeCun04,NeedellSrebroWard14,NeedellWard17,HLPR-COLT,dieuleveut2020}. Furthermore, such a burn-in time perspective is (in a sense) implicit in  approaches based on the ODE method of \cite{Ljung77} (see, e.g.,  the notion of asymptotic pseudotrajectories in~\cite{Benaim99}).

In the high-dimensional setting, however, it is less clear that one can ignore the burn-in period, as the dependence of its length on the dimension is poorly understood. 
Furthermore, in many problems of interest, 
random initializations are strongly concentrated in a neighborhood of an uninformative critical point of the population loss. Nevertheless, there is a wealth of numerical experiments showing that SGD may perform well in these regimes. This suggests that in high dimensions, the algorithm is able to recover from a random start on computationally feasible timescales, even in regimes where the behavior from a random start is quite different from that started in a trust region.

There has been a tremendous effort in recent years to
understand these issues and to develop general frameworks for 
understanding the convergence of SGD for non-convex loss functions in the
high-dimensional setting.
While many results still require convexity or quasi-convexity, several only require uniform control on 
derivatives of the empirical risk---such as $L$-smoothness (namely a uniform control on the Lipschitz constant of the gradient), and possibly  
uniform Lipschitzness---on a bounded domain.\footnote{When working in unbounded domains it is common to add a ``growth at infinity'' assumption.} 
These latter approaches typically study Monte Carlo analogues of SGD or develop a stochastic differential equation (SDE) approximation to it and control the rate of convergence of the corresponding processes to its invariant measure. 
Here one can obtain bounds that depend polynomially on the dimension and exponentially on quantities like $L$ and the radius of the domain $R$.  
See for example \cite{pmlr-v65-raginsky17a,pmlr-v65-zhang17b,cheng2018sharp,cheng2019stochastic,ma2019sampling} for a representative (but non-exhaustive) collection of works in this direction.\footnote{We note here that 
	some of these bounds depend polynomially on inverses of spectral quantities, such as the spectral gap of the Langevin operator or Cheeger constants. In general settings, these quantities often grow exponentially in the dimension.} 

In many basic problems of high-dimensional statistics, however, the assumptions of dimensionless bounds on both the domain and the derivative are unrealistic,
even in average case, due to the concentration of measure phenomenon.
For illustrative purposes, consider the simplest example of linear regression with Gaussian covariates. Here a simple calculation shows that the usual empirical risk---the least squares risk---is $L$-smooth and $K$-Lipschitz on the unit ball for $L$ and $K$ diverging linearly in $N$. 
Of  course it is possible to re-scale the loss so that $L=O(1)$,
but such a re-scaling will dramatically change
the invariant measure of natural Monte Carlo or SDE approximations and render it uninformative; the aforementioned bounds on the convergence rate are then no longer relevant for the estimation task. See Section~\ref{subsec:linear-regression} for more discussion on this, where we also consider the case of non-Gaussian covariates and errors.

Given the discussion above, we are led to the following questions which motivate our work here:
\begin{enumerate}
	\item Can one prove sample complexity bounds for SGD that are polynomial in the dimension for natural estimation tasks?
	\item For such tasks, what fraction of data is used/time is spent in the  search phase as opposed to the descent phase?
	\item How do these answers change as one varies the loss (or activation function)? Can this be answered using quantities that depend only on the \emph{population} loss?
\end{enumerate}

To understand these questions, we restrict our study to the simple but important high-dimensional setting of rank-one parameter estimation, though  we believe that it can be extended to more general settings.
This setting covers many important and classical problems  such as:
\begin{itemize}
	\item supervised learning for a single-layer network, and phase retrieval: \S\ref{subsec:supervised-learning}
	\item generalized linear models (GLMs) (e.g., linear and logistic regression): \S\ref{subsec:GLM}--\ref{subsec:linear-regression} 
	\item online PCA and spiked matrix and tensor models: \S\ref{subsec:online-PCA}--\ref{subsec:tensor-pca}
	\item two-component mixtures of Gaussians: \S\ref{subsec:mixture-of-Gaussians}. 
\end{itemize}
We treat all of these examples in detail in Section~\ref{sec:examples}.
For these examples, dimension-free smoothness assumptions do not apply just as in the case of linear regression described above. 
Furthermore for most of them, the loss function is non-convex and  
one expects exponentially many critical points, especially in the high entropy region \cite{BMMN17,BBCR18,MBAB19}. In spite of these issues, we prove that online SGD succeeds at these estimation tasks with polynomially many samples in the dimension.

Our main contribution is a classification of these rank-one parameter estimation problems on the sphere. This classification is defined solely by a geometric quantity,
called the \emph{information exponent} (see Definition~\ref{def:inf-exponent}), which captures the non-linearity of the \emph{population} loss near the region of maximal entropy (here, the equator of the sphere). 
More precisely, the information exponent is the degree of the first non-zero term in the Taylor expansion of the population loss about the equator. We study the dependence of the amount of data 
needed (i.e., the \emph{sample complexity}) for recovery of the parameter using online SGD on the information exponent. 
We prove thresholds for the sample complexity for online SGD to exit the search phase  which are linear, quasi-linear, and polynomial in the dimension depending on whether the information exponent is $1$, $2$, or at least $3$ respectively (Theorem~\ref{thm:main-1}). Furthermore, we prove sharpness of these thresholds up to a logarithmic factor in the dimension, showing that the three different regimes in this classification are indeed distinct (Theorem~\ref{thm:refutation-1}). 
Finally, we show that once the algorithm is in the descent phase, there is a law of large numbers for the trajectory of the distance to the ground truth about that for gradient descent on the population loss (Theorem ~\ref{thm:lln}). In particular, the final descent phase is at most linear in time for all such estimation problems and one asymptotically recovers the parameter.

As a consequence of our classification, we find that when the information exponent is at least $2$, essentially all of the data is used in the initial ``search phase'': the ratio of the amount of data used in the descent phase (linear in the dimension) to 
the amount used in the search phase (quasilinear or polynomial in the dimension) vanishes as the dimension tends to infinity. 
Put simply, the main difficulty in these problems is leaving the high entropy region of the initialization; once one has left this region, the algorithm's performance is fast in \emph{all} of these problems. This matches the heuristic picture described above. For an illustration of this discussion, see Figures~\ref{fig:search-phase}--\ref{fig:lln} for numerical experiments in the supervised learning setting.

Our classification yields nearly tight sample complexity bounds for the performance of SGD from a random start in a broad range of classification tasks. In particular, the setting $k\leq 2$ covers a broad range of estimation tasks for which the problem of sharp sample complexity bounds have received a tremendous amount of attention such as phase retrieval, supervised learning with commonly used activation functions such as ReLu and sigmoid, Gaussian mixture models, and spiked matrix models. The regime $k\geq 3$ has, to our knowledge, received less attention in the literature with one notable exception, namely spiked tensor models. That being said, we find fascinating phenomena in these problems. For example, in the supervised learning setting, we find that minor modifications of the activation function can dramatically change the sample complexity needed for consistent estimation with SGD. See, e.g., Fig.~\ref{fig:k=1-2-3}.

Importantly, our approach does not require almost sure or high probability assumptions on the geometry of the loss landscape, such as convexity or strictness and separation of saddle points. Furthermore, we make no assumptions on the Hessian of the loss. Indeed, in many of our applications the Hessian of the loss  where the algorithm is initialized will have many positive and negative eigenvalues simultaneously.
Instead we make a scaling assumption on
the sample-wise error for the loss. More precisely, we assume a scaling in $N$ for the low 
moments of the norm of its gradient and the variance of the directional derivative in the direction of the parameter to be inferred. 
We note that the assumption of $L$-smoothness falls into our setting, but our assumption is less restrictive. For a precise definition of our assumption and a comparison see Definition~\ref{def:naturally-scaling} below.
We expect that 
the classification we introduce has broader implications for efficient estimation thresholds than the setting we consider here,
such as finite rank estimation, or offline algorithms with possibilities of reuse, batching, and overparametrization. Furthermore, while we restrict to the sphere,  we expect a similar classification to hold in full space with e.g., an $\ell^2$ penalty, but leave this to future investigation.

\subsection{Formalizing ``search'' vs.\ ``descent'': weak and strong recovery}\label{subsec:weak-recovery}

One of our goals in this work is to understand the dimension dependence of the relative proportion of time spent by SGD in the search phase as opposed to the descent phase. As such, we need a formal definition of the search phase. 
To this end, we focus on the simple setting where the population loss is a (possibly non-linear) function of the distance to the parameter. Furthermore, to make matters particularly transparent, and for simplicity of working in a bounded domain, we will assume that the norm of the parameter is known.
Note that in some settings, fixing the norm amounts to assuming a fixed variance \cite{johnstone2000distribution,sur2019modern}.

More precisely, suppose that we are given a parametric family of distributions, $(P_x)_{x \in \bS^{N-1}}$,
of $\R^D$-valued random variables, parameterized by the unit sphere, $\bS^{N-1}$, 
and $M = \alpha N$ i.i.d.\ samples $(Y^\ell)$ from one of these distributions, $\prob_N = P_{\theta_N}$,
which we call the \emph{data distribution}. Our goal is to estimate $\theta_N$ given these samples, via online SGD with loss function $\cL_N:\mathbb{S}^{N-1}\times\R^D\to\R$.
We study here the case where the \emph{population loss} is of the form
\begin{equation}
\Phi_N(x):=\mathbb E_N[\cL_N]= \phi(m_N(x))\qquad\mbox{where}\qquad  m_N(x)=x\cdot \theta_{N}\,,\label{eq:Phi-form}
\end{equation}
for some $\phi: [-1,1]\to \mathbb R$ (here $\cdot$ denotes the Euclidean inner product on $\mathbb{R}^{N}$). 
We call $m_N(x)$ the \emph{correlation} of $x$ with $\theta_N$. 
We often also refer to $m_N$ as the latitude, and call the set $m_N(x)\approx 0$ the \emph{equator} of the sphere.  

In order to formalize the notion of exiting the ``search phase'', 
we recall here the notion of ``weak recovery",
i.e., achieving macroscopic correlation with $\theta_N$.
We say that a sequence of estimators $\hat \theta_N\in \mathbb{S}^{N-1}$ \emph{weakly recovers} the parameter $\theta_N$
if for some $\eta>0$, 
\[
\lim_{N\to\infty} P\Big(m_N(\hat \theta_N ) \geq \eta\Big) = 1\,.
\]
{As $\|m_N\|_{\infty}\le 1$, this definition is equivalent to the existence of $\eta>0$ such that $$\liminf_{N\to \infty}\mathbb E[m_N(\hat\theta_N)] \ge \eta\,;$$
	this latter formulation was used in, e.g.,~\cite{mondelli2018fundamental}}.
To understand the scaling here, i.e.,
$\hat \theta_N \cdot \theta_N =\Theta(1)$, recall the basic fact from high-dimensional probability \cite{Vershynin}, 
that if $\hat\theta_N$ were drawn uniformly at random, then  $\hat\theta_N\cdot\theta_N\simeq N^{-1/2}$ and that the probability  of weak recovery with a random choice is exponentially small in $N$. 
In the context we consider here,
attaining weak recovery corresponds to exiting the search phase. 
On the other hand, our final goal is to understand  \emph{consistent estimation} or \emph{strong recovery} which in this setting amounts to showing that 
$m_N(\hat\theta_N) \to 1$
in probability {(or equivalently in $L^p$-norm for $p\geq 1$).}

\subsection{Algorithm and assumptions}
As our parameter space is spherical, we will consider a spherical  
version of online SGD which is defined as follows.
Let $X_{t}$ denote the output of the algorithm at time $t$, and let $\delta>0$
denote a step size parameter.
The sequence of outputs of the algorithm are then 
given by the following procedure:
\begin{equation}
\begin{cases}
X_{0}= & x_{0}\\
\tilde{X}_{t}= & X_{t-1} - \frac{\delta}{N}\nabla\cL_N(X_{t-1};Y^{t})\\
X_{t}= & \frac{\tilde{X}_{t}}{\norm{\tilde{X}_{t}}}
\end{cases}\,,\label{eq:Algorithm-def}
\end{equation}
where the initial point $x_{0}$ is possibly random, $x_{0}\sim\mu\in\mathcal{M}_{1}(\mathbb{S}^{N-1})$
and where $\nabla$ denotes the spherical gradient, {i.e., for a function $f: \mathbb S^{N-1} \to \mathbb R$, 
	\begin{align*}
	\nabla f  = Df - \frac{\partial f}{\partial r} \frac{\partial}{\partial r}\,,
	\end{align*}
	where $D$ is the derivative in $\mathbb R^N$ and $\frac{\partial}{\partial r}$ is the partial derivative in the radial direction, again in $\mathbb R^N$.} 
In the online setting, we terminate the algorithm
after it has run for $M$ steps (though in principle one could terminate earlier).
We take, as our estimator, the output of this algorithm. 
In order for this algorithm to be well-defined, we assume that the loss is almost surely differentiable in the parameter for all $x\in\bS^{N-1}$.

As we are studying a first-order method, it is also natural to expect that the output of gradient flow is a Fisher consistent estimator for all initial data with positive correlation, meaning that it is consistent when evaluated on the \emph{population} loss. To this end, we say that \textbf{Assumption A} holds if $\phi$ is differentiable and $\phi'$ is strictly negative on $(0,1)$. Observe that Assumption A holds if and only if gradient flow for the population loss eventually produces a consistent estimator, when started anywhere on the upper half-sphere $\{x:m_N(x)>0\}$\footnote{{Observe that the gradient flow on $\Phi$ reduces to its projection on $m_N$, which solves an autonomous  ODE; Assumption~A holds if and only if the window $(0,1)$ is in the absorbing set of a minimum of $\phi$ at $m_N = 1$.}}. The reason we restrict this assumption to $(0,1)$ and the upper half-sphere is to include problems where the parameter is only recovered up to a sign due to an inherent symmetry in the task. We emphasize here that Assumption A is a property only of the population loss, and  does not imply convexity of the loss $\cL_N$. {(We note however, that this is an assumption on the (unknown) data distribution, and cannot be empirically verified.)}

In order to investigate the performance of SGD in a regime that captures a broad range of high-dimensional inference tasks, we need to choose a scaling for the fluctuations of the loss. These fluctuations are captured by the \emph{sample-wise error}, defined by
\[
H_{N}^{\ell}(x) : =\cL_N(x;Y^{\ell})-\Phi_N(x)\,.
\]
We seek a scaling regime which does not suffer from the issues described in the introduction. To this end, 
we work under the following assumption which is satisfied by the loss functions of many natural high-dimensional problems.

\begin{defn}\label{def:naturally-scaling}
	For a sequence of data distributions and losses $(\mathbb P_N,\cL_N)$, we say that 
	\textbf{Assumption~B} holds if there exists $C_{1},\iota>0$ such that the
	following two moment bounds hold 
	for all~$N$: 
	\begin{enumerate}[(a)]
		\item We have that
		\begin{equation}
		\sup_{x\in\mathbb{S}^{N-1}}\mathbb{E}\big[(\nabla H_N(x)\cdot \theta_N)^{2}\big]\leq  C_1\,,\label{eq:condition-1}
		\end{equation}
		\item and that, 
		\begin{equation}
		\sup_{x\in\mathbb{S}^{N-1}}\E[\norm{\nabla H_N(x)}^{4+\iota}]<C_{1}N^{\frac{4+\iota}{2}}.\label{eq:condition-2}
		\end{equation}
	\end{enumerate}
\end{defn}

This assumption captures the scaling regimes commonly used for high-dimensional analyses of statistical problems throughout the literature, see, e.g., \cite{johnstone2000distribution,wainwright09,MR14,CandesLS14}.
For an in-depth discussion, see \prettyref{sec:examples}, where we show that a broad class of statistical models satisfy Assumption B. 
Observe that the scaling relation between (a) and (b) is tight when $\nabla H_N$ is an i.i.d.\ sub-Gaussian vector.  
On the other hand, note that if $\cL_N$ is $L$-smooth on the unit sphere for some fixed $L=O(1)$, then Assumption B holds,
with an $O(1)$ bound instead of an $O(N^{2+\iota/2})$ bound in \eqref{eq:condition-2}. In particular, Assumption B applies more generally than $O(1)$-smoothness.

\subsection{Main results}
In this paper, we show that a key quantity governing the performance of online SGD
is the following, which we call the information exponent for a population loss.
\begin{defn}\label{def:inf-exponent}
	We say a population loss $\Phi_N$ has \textbf{information
		exponent $k$} if $\phi\in C^{k+1}([-1,1])$ and there exist
	$C,c>0$ such that 
	\begin{equation}
	\begin{cases}
	\frac{d^{\ell}\phi}{dm^{\ell}}(0)=0 & 1\leq\ell<k\\
	\frac{d^{k}\phi}{dm^{k}}(0)\leq-c<0\\
	\norm{\frac{d^{k+1}\phi}{dm^{k+1}}(m)}_\infty \leq C
	\end{cases}\,.\label{eq:pop-loss-assumption}
	\end{equation}
\end{defn}
\noindent We compute the information exponent for a broad class of examples in Section~\ref{sec:examples}.
(If the first non-zero derivative is instead positive, then Assumption A cannot hold, as $\phi'(\varepsilon)$ will be positive for $\varepsilon>0$ sufficiently small.)

Our first result upper bounds the sample complexity for
consistent estimation. In the sequel, let 
\begin{equation}
\alpha_c(N,k) =
\begin{cases}
1 & k = 1 \\
\log N & k = 2 \\
N^{k-2} & k \geq 3
\end{cases}
\end{equation}
and say that a sequence $x_n\ll y_n$ if $x_n/y_n\to0$.
For concreteness, we state our result in the case of random initialization, namely we take $\mu_N$ to be the uniform measure conditioned on the upper half sphere $\{m(x)\geq0\}$. (This conditioning is without loss of generality up to a probability $1/2$ event and is introduced to avoid obvious symmetry issues: see Remark~\ref{rem:starting-in-upper-half-sphere}.)   
We then have the following.
\begin{thm}\label{thm:main-1}
	Suppose that Assumptions A and B hold and that the population loss has information exponent $k$. Let $M=\alpha_N N$ with $\alpha_N$ growing at most polynomially in $N$. 
	If $(\alpha_N,\delta_N)$ satisfy
	$\alpha_N^{-1}\ll \delta_N \ll \alpha_N^{-1/2}$ and 
	\begin{itemize}
		\item \emph{($k=1$:)} $\alpha_N \gg \alpha_c(N,1)$
		\item \emph{($k=2$:)} $\alpha_N \gg \alpha_c(N,2)\cdot\log N$
		\item \emph{($k\geq3$:)} $\alpha_N\gg\alpha_c(N,k)\cdot (\log N)^2$
	\end{itemize}
	then 
	online SGD with step size parameter $\delta_N$ started from $X_0\sim\mu_N$, will have
	\[
	m_N(X_M)\to 1\,,\qquad \mbox{in probability, {and in $L^p$ for every $p\ge 1$}}\,.
	\]
\end{thm}

Our second result is the corresponding lower bound on the sample complexity required for exiting the search phase with a given information exponent. 
\begin{thm}\label{thm:refutation-1}
	Suppose that Assumptions A and B hold and that the population loss
	has information exponent $k\geq1$. 
	If $(\alpha_N,\delta_N)$ are such that
	\begin{itemize}
		\item \emph{\textbf{($k=1,2$):}} $\alpha_N  \ll \alpha_c(N,k)$ and $\delta_N=O(1)$
		\item \emph{\textbf{($k\geq 3$):}} $\alpha_N \ll \alpha_c(N,k)$ and  $\delta_N=O(\alpha_N^{-1/2})$,
	\end{itemize}
	then the online SGD with step size parameter $\delta_N$, started from $X_0 \sim \mu_N$, will have
	\begin{align*}
	\sup_{t\leq M} |m_N(X_t)| \to 0\,,\qquad \mbox{in probability, {and in $L^p$ for every $p\ge 1$}}\,.
	\end{align*}
\end{thm}

Let us pause to comment on the interpretation of these results.
The first result states that the sample complexity of consistent estimation for a problem with finite information exponent is always at most polynomial, and provides a precise scaling for this polynomial, $\alpha_c(N,k)$, as a function of the information exponent $k$. The second result says that the thresholds $\alpha_c(N,k)$ are optimal up to $O((\log N)^2)$. We expect, in fact, that the thresholds $\alpha_c(N,k)$---without additional logarithmic factors---are sharp; see Section~\ref{subsec:discussion-results} for more on this.  Observe that the second result in particular implies the 
algorithm can \emph{only} exit the search phase when the number of samples is at least $\alpha_c(N,k) N$. {Finally, notice that while the recovery result Theorem~\ref{thm:main-1} specified a window of feasible learning rates $\delta_N$ for these guarantees to apply, the refutation result covers a much wider range of $\delta_N$.} 

{Our arguments will also show that the ratio of the number of samples used in the descent phase to the number used in the search phase is $O(\alpha_c(N,k)^{-1})$ which is vanishing for $k\geq2$.
	More precisely, this observation follows from Theorem~\ref{thm:refutation-1} together with the following. Let $\tau_{\eta}^+$ denote the first $t$ such that $m_N(X_t)>\eta$ and let $\tau_{1-\eta}^+$ denote the first $t$ such that $m_N(X_t) >1-\eta$.}

\begin{thm}\label{thm:main-2}
	Suppose that Assumptions A and B hold and that the population loss has information exponent $k\geq2$. Let $M=\alpha_N N$ with $(\alpha_N,\delta_N)$ as in Theorem~\ref{thm:main-1}. For any $\eta>0$ there is a constant $C= C(k,\eta)>0$ such that $\tau_\eta^+ \gg \alpha_c(N,k)$
	and $|\tau_{1-\eta}^+ -\tau_\eta^+|\leq CN$ with probability $1-o(1)$. Furthermore, $X_t>1-2\eta$ for all $ \tau_{1-\eta}^+\leq t\leq M$ with probability $1-o(1)$.
\end{thm}

{In words Theorem~\ref{thm:main-2} says that most of the data is used in the search phase (i.e., to attain some non-trivial correlation), and that from there descent to essentially full correlation is rapid, and takes $O(N)$ samples independently of the class of the problem.}

\begin{rem}\label{rem:k=1}
	In the case that $k=1$, we show consistent estimation and its refutation only in the regimes $\alpha\to\infty$ and $\alpha\to0$ respectively. That this is optimal can be seen in the simple example of estimating the mean of an i.i.d. Gaussian vector, where consistent estimation is information theoretically impossible with $\alpha = O(1)$.\footnote{Suppose that we are given $M=\alpha N$ samples of an $N$-dimensional Gaussian vector with law $\cN(v,Id)$, where $v$ is a unit vector and $\alpha$ is a fixed constant. It is easy to see that for large $N$, the sample average only achieves a (normalized) inner product of $v\cdot\frac{\hat v}{\norm{v}}\to\sqrt{\alpha(1-\alpha)^{-1}}<1$ when $\alpha=O(1)$. Furthermore, by the Cramer-Rao bound, this is tight for unbiased estimators satisfying a second moment constraint.}
	If one only considers \emph{weak} recovery, one can sharpen the thresholds in $\alpha$ to the $O(1)$ scale: see Theorem~\ref{thm:k=1}.
\end{rem}

\begin{rem}\label{rem:phi-N-dependence}
	All of these results hold in the broader setting that $\phi=\phi_N$ varies in $N$ provided the following generalizations of Assumption A and Definition~\ref{def:inf-exponent} hold. We take Assumption A to be that $\phi'_N$ are equi-continuous and uniformly negative on $(0,1)$ and Definition~\ref{def:inf-exponent} to be that~\eqref{eq:pop-loss-assumption} holds uniformly over the sequence. 
\end{rem}

\subsection{Discussion of the main results}\label{subsec:discussion-results}
Let us now discuss the intuition behind the information exponent and Theorems~\ref{thm:main-1}--\ref{thm:refutation-1}. 
Together, Theorems \ref{thm:main-1}-\ref{thm:refutation-1}
show that the information exponent governs, in a sharp sense, the performance of online SGD
when started from a uniform at random initializations. (In fact, as we will see in Theorems \ref{thm:main-theorem}-\ref{thm:lln}, this is a more general phenomenon for all initializations starting near the equator.) 

To understand where these thresholds come from, consider the following simplification of the recovery problem given by 
\begin{equation}\label{eq:population-dynamics}
m_t = m_{t-1} -\frac{\delta}{N} \phi'(m_{t-1})\norm{\nabla m_{t-1}}^2 \approx m_{t-1} + \frac{\delta}{N} c m_{t-1}^{k-1}\,, 
\end{equation}
for $m_{t-1}$ small and some $c>0$.
This amounts to gradient descent for the  population loss (omitting the projection as its effects are second order for small $\delta$, see Sections~\ref{subsec:controlling-radial-effects}--\ref{subsec:second-order-correction}.)  This corresponds to the ``best case scenario" since the observed losses will be corrupted versions of the population loss. One would expect this corruption to only increase the difficulty of recovery.

Analyzing the finite difference equation of~\eqref{eq:population-dynamics}, one finds that if the initial latitude is positive but microscopic, $m_0\asymp N^{-\zeta}$ for any $\zeta>0$, there are three regimes with distinct behaviors: 
\begin{itemize}
	\item $k < 2:$ the time for $m_t$ to weakly recover ($m_t \geq \eta$) is linear, i.e, order $\delta^{-1} N$.
	\item $k = 2:$ the time for $m_t$ to weakly recover is quasi-linear, i.e., order $\delta^{-1} N\log N$.
	\item $k> 2:$ the time for $m_t$ to weakly recover is polynomial:  $\delta^{-1}N^{1+c_\zeta}$ for $c_\zeta = (k-2)\zeta >0$. 
\end{itemize}
In the online setting, the number of time steps is equal to the number of samples used. 
Consequently, \emph{no matter what $\zeta\in (0,1)$ is}, there is a transition between linear sample complexity ($\alpha = O(1)$) and polynomial sample complexity ($\alpha = N^\zeta$ for some $\zeta>0$) regimes for the gradient descent on the population loss 
as one varies the information exponent,  through the critical $k_c=2$. The precise thresholds obtained in our results
correspond to the choice of $\zeta=1/2$, which is the scaling for uninformative initializations in high dimensions, e.g., the uniform measure on $\mathbb{S}^{N-1}$. 

When one considers the true online SGD, there is an effect due to the sample-wise error for the loss, $H_N$, whereby to first approximation, 
\begin{align*}
m_t \approx m_{t-1}+ \frac{\delta}{N} a_k m_{t-1}^{k-1} - \frac{\delta}{N} \nabla H_N^t (X_{t-1}) \cdot \theta_N\,.
\end{align*}
Due to the independence of $\nabla H_N^{t}$ and $X_{t-1}$, as we sum the contributions of the third term in time, we obtain a martingale which we call the \emph{directional error martingale},
\begin{equation}\label{eq:DE-martingale}
M_t = \frac{\delta}{ N}\sum_{j=1}^{t} \nabla H_N^j(X_{j-1})\cdot \theta_N\,.
\end{equation}
By  Doob's inequality, its cumulative effect can be seen to typically be of order $\delta \sqrt{T}/N$. In order for this term's contribution to be negligible for time scales on the order of $M$, and allow for recovery, we ask that it is comparable to $m_0$, dictating that $\alpha \delta^2=O(1)$. This relative scaling of $\alpha$ and $\delta$ is therefore fundamental to our arguments.
Indeed if $\alpha \delta^2$ were diverging, then on timescales that are of order $\alpha N$, the cumulative effect of projections becomes a dominant effect potentially drowning out the drift induced by the signal. We remark here that for the refutation result of Theorem~\ref{thm:refutation-1} when $k\geq 3$, if one only want to assume that $\delta =O(1)$, our arguments would show a refutation result for all $\alpha \leq N^{(k-2)/2}$, implying that for any $O(1)$ choice of $\delta$, the $k\ge 3$ regime still requires polynomial sample complexity.

We end this discussion by briefly comparing our approach to three related approaches that have been used to investigate similar questions in recent years.  An in-depth problem-by-problem discussion of the related literature will appear in Section~\ref{sec:examples} below. The classical approach to such problems is the ``ODE method'' which dates back at least to the work of  \cite{Ljung77}  (see~\cite{Benaim99} for a textbook introduction). Here one proves convergence of the trajectory with small step-size to the solution of the population dynamics. While this approach is most similar to our approach, it is designed for fixed dimensions. Indeed, in the scaling regime studied here, namely that corresponding to the initial search phase, one cannot neglect the effect of the directional error martingale, as its increments are \emph{larger} than those of the drift. 
Another approach to estimation problems via gradient-based algorithms is to study a diffusion approximation to this problem. Diffusion approximations to stochastic approximation algorithms date at least back to the work of  \cite{mcleish} in finite dimensions; more recently they have been studied in high-dimensions for online algorithms, e.g., \cite{online-ICA-NIPS,TanVershynin}. In our setting, rather than setting up a functional law of large numbers or central limit theorem for $m_t$, the precise nature and analyzability of which we expect depends on, e.g., the choice of activation function for GLMs, we use a system of difference inequalities similar to \cite{BGJ18a,BGJ18b}.

On the other hand, several statistical physics motivated approaches have recently been introduced to study such questions in both online \cite{WaLu16,WaMaLu16} and offline \cite{big-bad-minima,marvels-pitfalls} settings. These results develop a scaling limit, in the regime where one first sends $N\to\infty$ then $T\to\infty$. In these settings, however, the solution to the corresponding limiting evolution equation admits a trivial zero solution if $m_0 =0$ and a meaningful solution if $m_0>0$. This limiting system is then analyzed for its behavior and recovery thresholds, for various values of $m_0>0$ small, but uniformly bounded away from zero. As a consequence of our work, we find that if one can produce such initial data then the problem is always solvable with linear sample complexity,  and satisfies a scaling limit to the trajectory of the population dynamics. 
On the other hand, for uninformed initializations, where $m_0=o(1)$ except with exponentially small probability, the bulk of the data is used just to reach order 1 latitudes.

\begin{rem}\label{rem:starting-in-upper-half-sphere}
	Here and in the following, we have restricted attention to initializations supported on the upper half-sphere. 
	We do this to focus on the key issues and deal with general losses and data distributions  in a unified manner with minimal assumptions.
	For even information exponents, this restriction is without loss of generality by symmetry. For odd information exponents, when started from the lower half-sphere, the dynamics would rapidly approach the equator. We also restrict our attention to starts which have initial correlation on the usual CLT scale ($m(X_0)\sim N^{-1/2}$) which will hold for the uniform at random start. We expect that one can obtain similar results for all possible initializations if one imposes an anti-concentration assumption on the directional error martingale. Without such an  assumption, an initialization on the equator $m_0 = 0$ can be trapped for all time.
\end{rem}

\section{Applications to some important inference problems }\label{sec:examples}
In this section, we illustrate how our methods can be used to quantify recovery thresholds for online SGD in  inference tasks 
arising from various broad parametric families commonly studied 
in high-dimensional statistics, machine learning, and signal processing: supervised learning for single-layer networks with Gaussian features,  generalized linear models, linear regression, online PCA, spiked matrix and tensor models,
and Gaussian mixture models.  Each of these has a vast literature to which we cannot hope to do justice; instead we focus on describing how each satisfy the criteria of Theorems~\ref{thm:main-1}--\ref{thm:main-2}, emphasizing how the different information exponent regimes appear in variations on these problems, and relating the recovery thresholds we obtain for online SGD to past work on related algorithmic thresholds for these parameter estimation problems in the high-dimensional regime. {In many of our examples, we work only under mild moment assumptions on the data. As our goal here is to demonstrate the applicability of our classification, we do no try to optimize the number of moments assumed.} For ease of notation, we henceforth suppress the dependence of quantities on $N$ when it is clear from context.  

\subsection{Supervised learning for single-layer networks}\label{subsec:supervised-learning}
Consider the following model of supervised learning with a single-layer network.
Let $v_0\in\mathbb{S}^{N-1}$
be a fixed unit vector and suppose that we are given some (possibly non-linear) \emph{activation function} $f:\mathbb R \to \mathbb R$, some \emph{feature vectors} $(a^\ell)_{\ell = 1,...,M}$, and with these, $M$ noisy \emph{responses}
of the form
\begin{equation}\label{eq:glm}
y^{\ell}=f(a^{\ell}\cdot v_0) {+ \epsilon^\ell}\,.
\end{equation}
where $(\epsilon^\ell)^\ell$ are additive errors. 
Our goal is to estimate $v_0$ given this data. Our approach is by minimizing the least squares error. 

This model and special cases thereof have been studied under many different names
by a broad range of communities.
It is sometimes called a teacher-student network, a single-index model, or seen as a special class of generalized linear models (see, e.g.,~\cite{hastie2009elements,bishop,BKMMZ19}). This model has received a tremendous amount of attention in recent years,
largely in the regime that the features are taken to be i.i.d.\ standard Gaussians. Here various properties have been studied such as information theoretic thresholds \cite{BKMMZ19,mondelli2018fundamental}, the geometry of its landscape \cite{MBAB19,SQW18}, as well
as performance of various gradient-type and spectral methods \cite{lu2017phase,mondelli2018fundamental,LALspectral18}.

To place this in the framework of Theorem~\ref{thm:main-1}, we consider as data the pairs $Y^\ell = (y^\ell,a^\ell)$ for $\ell = 1,...,M$.
Our goal is then to optimize a quadratic loss $\mathcal L(\cdot ; Y)$
of the form 
\[
\cL(x;Y)=\cL(x;(y,a))=\Big(y-f\Big(a\cdot x)\Big)\Big)^{2}\,.
\]	
Note that we may write the population loss as
\begin{equation}\label{eq:GLM-pop-loss-1}
\Phi(x)=\E\Big[\Big(f\Big(a\cdot x)\Big)-f\Big(a\cdot v_0\Big)\Big)^{2}\Big]{ + \mathbb E[\epsilon^2]}\,.
\end{equation}
For general activation functions $f$ and distributions over the features $(a^\ell)$, one could proceed to calculate the information exponents by Taylor expansion.  

Let us focus our discussion on the most studied regime, namely where $(a^\ell)$ are i.i.d.\ standard Gaussian vectors in $\mathbb R^N$; {for the $(\epsilon^\ell)$ we only assume they are i.i.d.\ mean zero with finite $4+\delta$-th moment for some $\delta>0$}.
Here we find an explicit representation for the population loss which allows us to compute the information exponent for activation functions of at most exponential growth. With this we find that Theorems~\ref{thm:main-1}-\ref{thm:main-2}
apply to all such activation functions and, in particular, we find a wealth of interesting phenomena. 

Recall that the Hermite polynomials, which we denote by $(h_k(x))_{k=0}^\infty$, are the (normalized) orthogonal polynomials
of the Gaussian distribution $\varphi(dx)\propto\exp(- x^2/2)dx$.
Define the $k$-th Hermite coefficient for an activation function $f$ by
\begin{align}\label{eq:hermite-decomposition}
u_k(f) = \langle f,h_k \rangle_{L^2(\varphi)} = \frac{1}{\sqrt{2\pi}} \int_{-\infty}^{\infty} f(z)h_k(z) e^{ - z^2/2} dz\,.
\end{align}
As long as $f'$ is of at most polynomial growth---i.e., there exist $A, B\geq 0$ and integer $q\geq 0$ such that $|f'(x)|\leq A|x|^q+B$ for all $x$---the population loss is differentiable and the above exists. The population loss then has the following exact form which we believe is of independent interest. 

\begin{prop}\label{prop:GLM-pop}
	Suppose that the features $(a^\ell)$ are i.i.d.\ standard Gaussian vectors, {and the errors $(\epsilon^\ell)$ are i.i.d.\ mean zero, variance $C_\epsilon$ and with finite $4+\delta$-th moment for some $\delta>0$}.
	Suppose that the activation function, $f$, is differentiable {a.e.} and that $f'$ has at most polynomial growth. Then 
	\begin{equation}
	\Phi(x)=\phi_f(m(x))\qquad \mbox{where}\qquad \phi_f(m): = 2\sum_{j=0}^{\infty}(u_{j}(f))^{2}(1-m^{j}) + C_\epsilon \,,\label{eq:GLM-pop}
	\end{equation}
	and where $u_j$ are as in~\eqref{eq:hermite-decomposition}.
	Furthermore, Assumptions A and B hold.
\end{prop}

With Proposition~\ref{prop:GLM-pop}, it is easy to compute the information exponents corresponding to general activation functions.
The following result is an immediate consequence of \prettyref{prop:GLM-pop}.
\begin{cor}\label{cor:glm-ie}
	Suppose that the features $(a^\ell)$ are i.i.d.\ standard Gaussian vectors, {and the errors $(\epsilon^\ell)$ are i.i.d.\ mean zero with finite $4+\delta$-th moment for some $\delta>0$}. 
	Suppose that the activation function, $f$, is differentiable a.e. and that $f'$ has at most polynomial growth. 
	Then, 
	\begin{enumerate}
		\item The information exponent of $f$ is $1$ if and only if $u_1(f) \neq 0$.
		\item The information exponent of $f$ is $2$ if and only if $u_1(f) = 0$ and $u_2(f) \neq 0$
		\item The information exponent of $f$ is at least $3$ if and only if $u_1(f)=u_2(f) =0$.
	\end{enumerate}
\end{cor}

\noindent We prove Proposition~\ref{prop:GLM-pop} and Corollary~\ref{cor:glm-ie} in \prettyref{app:Generalized-linear-models}.

Let us now turn to some concrete examples of activation functions and their classification. 
Many commonly used activation functions have
corresponding population loss with information exponent $1$, for example:
\begin{itemize}
	\item Adaline: $f(x)=x$
	\item Sigmoid: $f(x)=\frac{e^{x}}{1+e^{x}}$
	\item ReLu: $f(x)=x\vee 0$.
\end{itemize}

\noindent The problems of smooth and non-convex phase retrieval ($f(x) = x^2$ and $f(x) = |x|$ respectively) are examples of models whose population loss has information exponent 2: see Section~\ref{subsec:phase-retrieval} below for a discussion.
More generally, we note that the Hermite polynomial of degree $k$ gives a simple example of an activation with information exponent $k$.

\begin{figure}[t]\label{fig:many-runs}
	\begin{subfigure}{.48\textwidth}
		\centering
		\begin{tikzpicture}
		\node at (0,0) {\includegraphics[width=\linewidth]{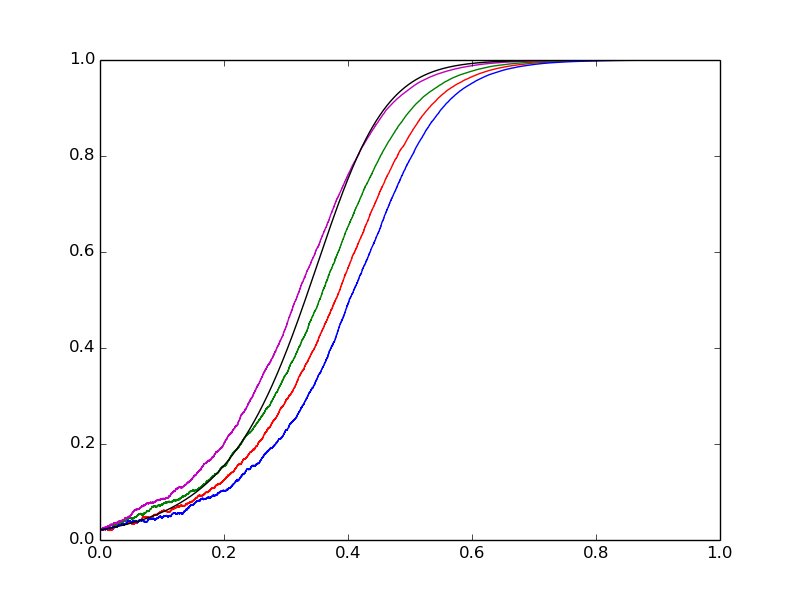}};
		\node at (-2.7,2.1) {$m$};
		\node at (2.7,-2.1) {$t/M$};
		\end{tikzpicture}
		\caption{$f(x) = x^2$, with $N=3000$ and $\alpha=100$.}
	\end{subfigure}
	\begin{subfigure}{.48\textwidth}
		\centering
		\begin{tikzpicture}
		\node at (0,0)
		{\includegraphics[width=1.04\linewidth]{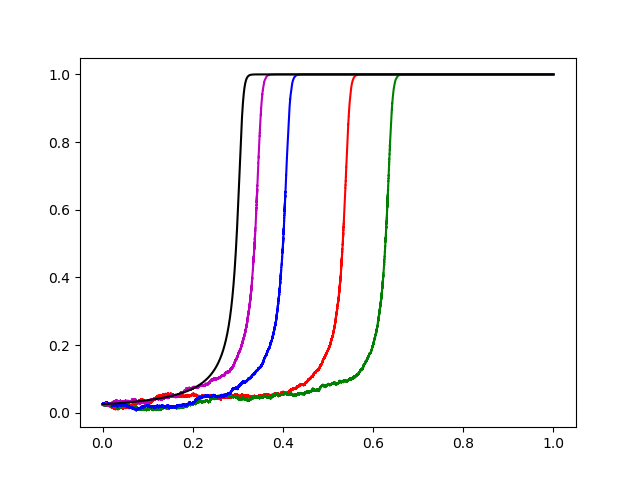}};
		\node at (-2.7,2.1) {$m$};
		\node at (2.85,-2.1) {$t/M$};
		\end{tikzpicture}
		\caption{$f(x)= x^3 - 3x$, $N=3000$, and $\alpha=30,000$.}
	\end{subfigure}
	\caption{(Random starts) In colors, 4-runs of single-layer supervised learning with activation functions $f(x) = x^2$ (phase retrieval) and $x^3 - 3x$ (3rd Hermite polynomial), initialized uniformly at random; in black, the corresponding population dynamics. Almost all of the data is used to attain a macroscopic correlation (the search phase); the trajectory through this phase does not concentrate (c.f.\ Fig.~\ref{fig:lln}).}\label{fig:search-phase}
\end{figure}

Let us now turn to discussing the implications of our results and in particular, the answers to the last two motivating questions from the introduction in this setting.

Our first observation is that
seemingly innocuous changes to an activation function can yield dramatic changes to the sample complexity of the problem.  For example, Adaline $f(x)=x$ and its cubic analogue $f(x)=x^3$ both have an information exponent of $1$; however, a linear combination of the two, namely $f(x)=x^3-3x$
has an exponent of $3$. Thus while the first two require linearly many samples to recover the unknown vector, the third requires $\Omega(N^{3/2}\log N)$ many samples just to exit the search phase. As an illustration of this, see Fig.~\ref{fig:k=1-2-3} where we perform supervised learning via SGD with the same pattern vectors and same unknown $v_0$ but different activation functions, namely $f(x)=x^3$ and $f(x)=x^3 - 3x$.

On the other hand, note that in the descent phase of the algorithm, where the algorithm exhibits a law of large numbers by Theorem~\ref{thm:lln}, one finds that the performance no longer depends on the activation function in a serious way. Thus from a ``warm start'', the choice of activation is less important. See Figs.~\ref{fig:search-phase}--\ref{fig:lln} for the example of the performance of phase retrieval, $f(x)=x^2$, and  $f(x)=x^3-3x$ with warm and random starts.  Let us emphasize here, however, that as soon as the information exponent is at least $2$, we know by Theorem~\ref{thm:main-2}, that almost all of the run-time is in the search phase.

\begin{rem}[Mis-specification]
	As another example of how the loss can dramatically affect performance, consider the problem of model mis-specification. Here we attempt to fit to the data using a possibly incorrect activation $g$. That is, our loss is $\cL=(y-g(a\cdot x))^2$, but our data is $y=f(a\cdot x)$.  One can compute the information exponent by noting that the loss satisfies $$\phi(m)=-2\sum_k u_k(f)u_k(g)m^k+\norm{f}^2+\norm{g}^2\,.$$
	Here we find various phenomena (no recovery, only weak recovery, or strong recovery) depending on the alignment of the Hermite coefficients of $f$ and~$g$. 
	
	For concreteness, let us consider the case of $$f(x)=u_{2}h_{2}(x)+u_{4}h_{4}(x) \qquad \mbox{and}\qquad g(x)=v_{1}h_{1}(x)+v_{2}h_{2}(x)+v_{4}h_{4}(x)\,.$$ In this case if we let $a=2u_{2}v_{2}$ and $b=2u_{4}v_{4}$, we have $\phi(m)=-am^{2}-bm^{4} + C$. Here we have the following phenomenology. 
	
	{
		\smallskip\noindent
		\textbf{Strong recovery.} Suppose that, $a>0$, and $a>-2b$, then this model satisfies Assumptions $A$ and $B$ and has information exponent $2$. Thus strong recovery has quasilinear sample complexity by Theorem~\ref{thm:main-1}. On the other hand, if either $f$ or $g$, has vanishing second Hermite coefficient, so that $a=0$, while $b>0$, the model has information exponent $4$ and the sample complexity is now cubic (up to log factors).
		
		\smallskip\noindent
		\textbf{Weak recovery.} Consider for concreteness the case that, $a=1$ and $b=-1$. Then Assumption B still holds. On the other hand, Assumption A only holds in a neighborhood of the origin. In particular, the global minimum of $\phi$ on $[0,1]$ is attained at $1/\sqrt{2}$. Thus by Theorem~\ref{thm:main-theorem} and a minor modification of  Theorem~\ref{thm:lln}, we see that SGD will weakly recover on quasi-linear time scales, will rapidly descent to latitude $m \approx 1/\sqrt{2}$, but then remain there on polynomial timescales, i.e., it will weakly recover but not strongly recover~$v_{0}$. 
		
		\smallskip\noindent
		\textbf{No recovery.} Suppose that and $a<0$. This would correspond to mis-specifying the sign of the coefficients of $f$. In this case $\phi=0$ is a local minimum of $\phi$. As such, a modification of the arguments of Theorem~\ref{thm:refutation-1} shows that SGD will not attain a macroscopic correlation in polynomial time. This is to be expected as the inference procedure we are applying is no longer Fisher consistent. On the other hand, consider the more extreme case of $v_1 = u_{2}=1$ and $u_{4}=v_{2}=v_{4}=0$. Here we try to fit the quadratic transformation of the data, $f(x)=x^{2}-1$, with a linear one, $g(x)=x$. In this case, a=b=0 so that $\phi(m)=0$, i.e., the population loss is constant and equal to 0. In this case even the population loss is completely uninformative and the algorithm will simply wander around the sphere, never attaining non-trivial correlation.}
\end{rem}

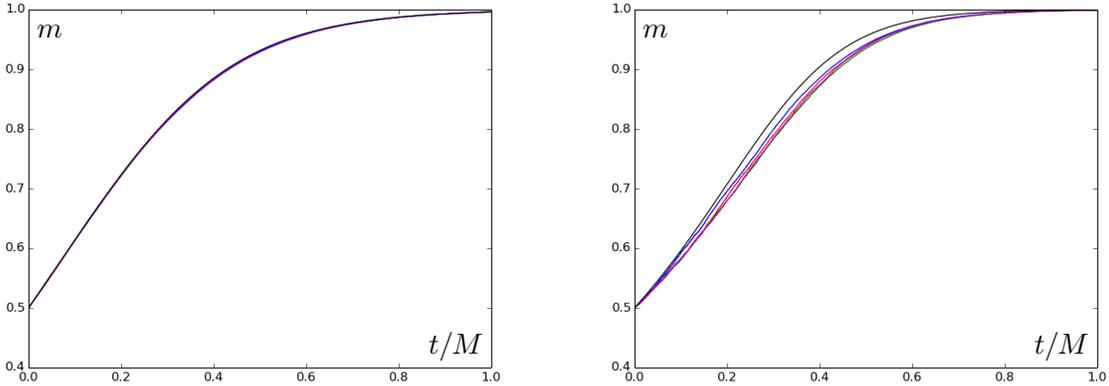
\begin{figure}[t]
	\begin{subfigure}{.48\textwidth}
		\centering
		\begin{tikzpicture}
		\node at (0,0) {
			\includegraphics[width=\linewidth]{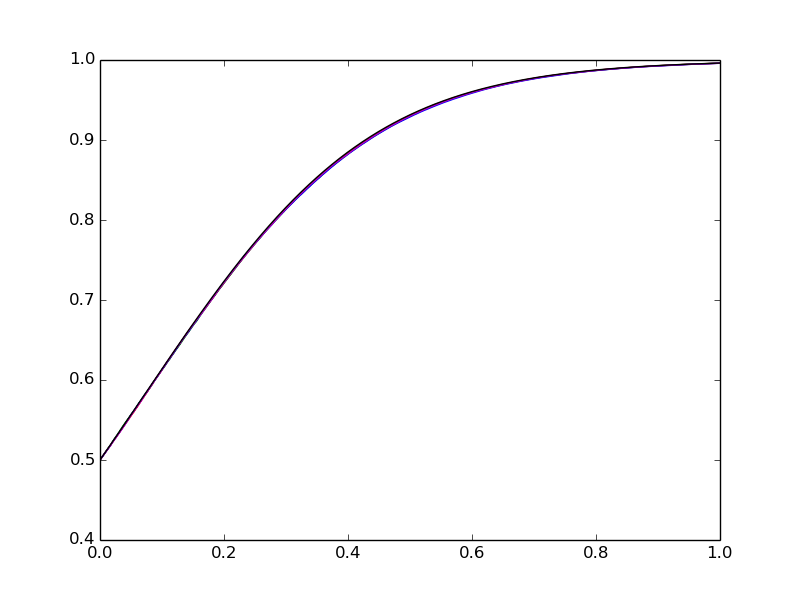}};
		\node at (-2.7,2.1) {$m$};
		\node at (2.7,-2.1) {$t/M$};
		\end{tikzpicture}
		\caption{$f(x) = x^2$, with $N=3000$ and $\alpha=500$.}
	\end{subfigure}
	\begin{subfigure}{.48\textwidth}
		\centering
		\begin{tikzpicture}
		\node at (0,0) {
			\includegraphics[width=\linewidth]{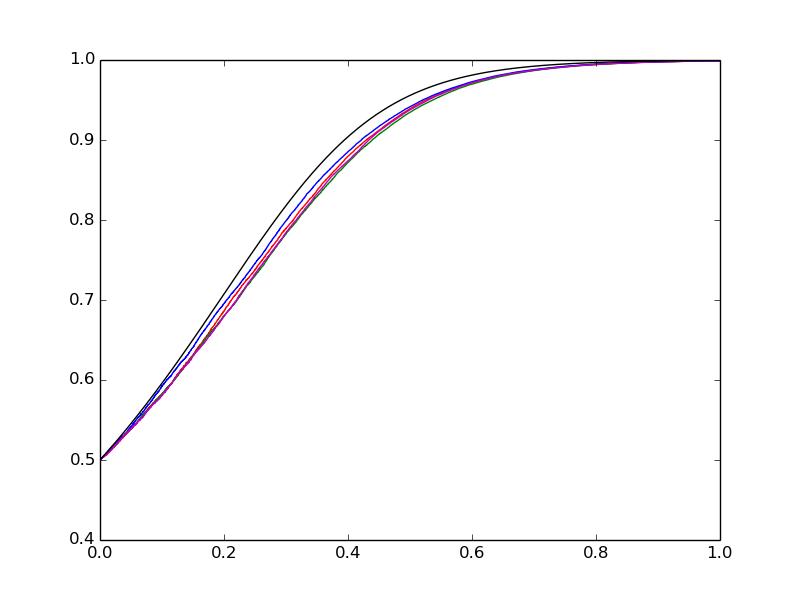}};
		\node at (-2.7,2.1) {$m$};
		\node at (2.7,-2.1) {$t/M$};
		\end{tikzpicture}
		\caption{$f(x)= x^3 - 3x$, $N=3000$, and $\alpha=500$.}
	\end{subfigure}
	\caption{(Warm start) In color, 4 runs of single-layer supervised learning, initialized from $m_0 = 0.5$, together with the corresponding population dynamics. Initialized here, and thus in the descent phase, the SGD rapidly converges (with linearly many data samples) to the ground truth, independently of the information exponent. Moreover, the trajectories 
		are well-concentrated about the population dynamics' trajectory. }\label{fig:lln}
\end{figure}

\subsubsection{Phase retrieval}\label{subsec:phase-retrieval}
One class of models of the form of~\eqref{eq:glm} that has received a tremendous amount of attention in recent 
years is smooth and ``non-smooth" phase retrieval.
These correspond to the cases $f(x) = x^2$ and $f(x)=\abs{x}$ respectively. Algorithmic recovery results with different algorithms, initializations, and variants of phase retrieval have been established in a wide array of related settings:~\cite{CCFM,CandesLS14,LeSuNIPS16,lu2017phase,TanVershyninKaczmarz,ALBZZ19,JeongGunturk,SQW18}.

For our results, in both cases we have that $u_{0},u_{2}>0$ and $u_{1}=0$ so that their
information exponents are both $k=2$.
As a consequence, Theorem~\ref{thm:main-theorem} shows that if $\alpha/(\log N)^{2}\to\infty$,
then online SGD with step size $\delta\sim\log N/\alpha^{2}$
will solve the weak recovery problem started from the uniform measure
conditioned on the upper-half sphere. Going further, by symmetry of
$f$, the same result holds started from the uniform measure on $\mathbb{S}^{N-1}$
itself, if we wish to weakly recover $v_{0}$ only up to a net sign.

In recent independent work, ~\cite{TanVershynin} 
have obtained sharper results for non-smooth phase retrieval.
There they show that a similar recovery result holds as soon as $\alpha$ is order $\log N$, and uniformly over all possible initializations, in particular those with $m(x) = 0$ (c.f.,~Remark~\ref{rem:starting-in-upper-half-sphere})
Furthermore, their work applies to online learning in ``full space'',
i.e., without restricting to the sphere. Some of the techniques there are similar 
in spirit to ours, but they involve a careful analysis of a 2D dynamical system which 
ends up being exactly solvable due to the choice of activation function, whereas we reduce to differential inequalities to handle general choices of activation function.

\begin{figure}[t]
	\centering
	\begin{tikzpicture}
	\node at (0,0) {
		\includegraphics[width=.65\linewidth]{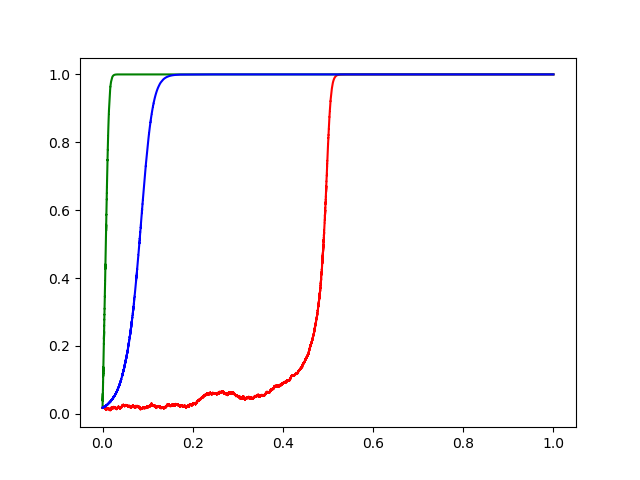}};
	\node at (-4.3,3.2) {$m$};
	\node at (4.75,-3.4) {$t/N$};
	\end{tikzpicture}
	\caption{SGD trajectories (at $N = 3000$ and $\alpha = 30,000$) from a  uniform at random start using the same data, but different activation functions $f(x) = x^3$ (green), $f(x) = x^2$ (blue) and $f(x) = x^3 - 3x$ (red), with information exponents $1,2,3$ respectively. The choice of activation function (changing the information exponent) can dramatically change the timescale for consistent estimation.}\label{fig:k=1-2-3}
\end{figure}

\subsection{Generalized linear models}\label{subsec:GLM}
Generalized linear models (GLM's)
are a commonly used family of statistical models which unify a broad class of regression tasks such as linear, logistic, and Poisson regression.
For a textbook introduction, see \cite{mccullagh,bishop}.  We follow the presentation of \cite{mccullagh}. (For ease of exposition, we focus only on the case of canonical link functions and  constant dispersion.) 

Given  an observation $y\in \R$ and a covariate vector $a\in \R^N$,
a generalized linear model consists of three components:
\begin{enumerate}
	\item The \emph{random component}: Conditionally on the 
	covariate $a$, the observation $y$ is distributed according to an element of the exponential family
	with canonical parameter~$\theta$, 
	\begin{equation}\label{eq:GLM-1}
	p_\theta(y) = \exp\Big( \frac{y \theta - b(\theta)}{d}- c(y)\Big)
	\end{equation}
	for some real-valued functions $b,c:\R\to\R$ and  some constant $d$.
	\item The \emph{systematic component}: The covariate $a$ and the unknown parameter $x\in \R^N$
	form a linear predictor, $\eta = a \cdot x.$
	\item The \emph{link}  between the random and the systematic components is given by 
	an \emph{activation function}, $f:\R\to\R$, which is monotone increasing
	and invertible,
	and satisfies
	\[
	f(\eta) = \mathbb E[y\mid a].
	\]
\end{enumerate}
For clarity, we work here in the case of a \emph{canonical link}, where we assume that the 
canonical parameter and the linear predictor are equal, i.e., $\theta =\eta$. 
In this case, as by definition $b$ is the cumulant generating function of $y\vert a$,
we have that $f(\eta) = b'(\eta)$  The goal is then to infer the parameter $x$ 
given the observation $y$.
Note that standard regression tasks can be mapped to this setting by choosing $b,c$ and $d$ appropriately. For example,  (in these examples we take $d=1$.)
\begin{itemize}
	\item Linear regression (linear activation  function $f(x)=x$ and $c(t)=b(t)=-t^2/2$.)
	\item Logistic regression (sigmoid activation function $f(x)=(1+e^{-x}))^{-1}$,  and $b(t)=\log 1+e^t$, $c(t)=0$.)
	\item Poisson regression (exponential activation function $f(x)=e^x$, and $b(t) = t$, $c(t)=- \log t!$)
\end{itemize}

Consider a random instance of this estimation problem. 
Let $v_0\in \bS^{N-1}$ be a fixed, unknown parameter.
Suppose that we are given $M$ i.i.d.\ covariate vectors $(a^\ell)_{\ell=1}^M$  and $M$ corresponding observations $(y^\ell)$.
Our goal is to estimate $v_0$ given $(y^\ell)$ and $(a^\ell)$. 

To place this in the framework of our results, we augment the observations by the covariates and consider the pairs 
$Y^\ell = (y^\ell,a^\ell)$. The canonical approach to these problems is maximum likelihood estimation, which in our setting
amounts to minimizing the loss
\[
\cL(x;Y) = \cL(x;y,a)= - y \,a\cdot x +b(a\cdot x) .
\]
For general activation functions $f$ and distributions over the covariates $(a^\ell)$, one proceeds to calculate the information exponents by Taylor expanding the population loss $\Phi(x) =\E\Phi(x;Y)$.

{To make things concrete, let us focus on the setting where the covariates are i.i.d.\ Gaussian vectors in $\mathbb R^N$. Here GLM's have been found to have a rich phenomenology in high dimensions, see, e.g., ~\cite{sur2019modern} for the case of logisitic regression.}
Here for a function $f$,  let $u_1(f) = \E [Z f(Z)]$ where $Z \sim a \cdot e_1$. We say that a function $f$ is of at most exponential growth if there are constants $A,B>0$ such that $\abs{f(x)}\leq A\exp(B\abs{x})$. 
We then have the following.
\begin{prop}\label{prop:GLM-1-prop} Let $f$ be an invertible, strictly increasing activation function of at most exponential growth for a GLM with standard Gaussian covariates $(a^\ell)_{\ell=1}^{M}$. Suppose that $f$
	is differentiable a.e. and that the exponential family \eqref{eq:GLM-1} has finite $8+\iota$'th moment for some $\iota>0$. Then the population loss, $\Phi(x)$, satisfies
	\begin{equation}\label{eq:GLM-rep}
	\Phi(x) = - u_1(f) m + C
	\end{equation}
	for some constant $C\in \R$. Furthermore, $u_1(f)>0$ so that the information exponent is $k=1$ and Assumptions A and B hold. 
\end{prop}
\noindent By Proposition~\ref{prop:GLM-1-prop}, maximum likelihood estimation for generalized linear models with invertible, increasing activations as above
will always have information exponent $k=1$. The proof of Proposition~\ref{prop:GLM-1-prop} is deferred to Appendix~\ref{app:GLM}. 

\subsection{{Linear regression with random covariates}}\label{subsec:linear-regression}
Let us return to the example of linear regression with random covariates discussed in the introduction.
Suppose that we are given i.i.d. data points of the form $\{(y^\ell,a^\ell)\}_{\ell=1}^M$,  
where $y^\ell = a^\ell\cdot v+\epsilon^\ell$, for some unknown unit vector $v$ and additive noise $\epsilon^\ell$. 
Consider the squared-error loss $\cL(x;y^\ell, a^\ell)=(y^\ell-a^\ell\cdot x)^2$. 
In this case, provided the covariates and errors are centered and have finite second moment, the population loss is of the form
\begin{equation}\label{eq:lin-reg-phi}
\Phi(x) = \norm{x-v}^2+1 = 2-2 m(x)+ C,
\end{equation}
for $x$ on the unit sphere.  
{
	
	\begin{prop}\label{prop:linear-regression-prop}
		Consider linear regression with centered i.i.d.\ covariates $(a^\ell)_{\ell}$  whose entries have finite $10$-th moment and centered i.i.d.\ errors $(\epsilon^\ell)_{\ell}$ with finite 5-th moment.  Then Assumptions A and B hold and the population loss has information exponent $1$. 
\end{prop}}
{Proposition~\ref{prop:linear-regression-prop} is proved in Appendix~\ref{app:linear-regression}. We emphasize here that we do not assume that the covariate vectors have independent entries, nor do we assume Gaussianity of either the covariates or the errors.}

{
	\begin{rem}\label{rem:linear-regression-scaling}
		We now return to the issue raised in the introduction regarding the scaling of the risk and its derivatives. 
		For simplicity let us focus on the case where the designs and noise are i.i.d. standard Gaussian.
		The canonical empirical risk in this setting is the least squares risk $\hat R(x)= \norm{Y-Ax}_2^2$. 
		To apply the SDE approximations referenced above  to the SGD we consider, one should work with this risk.
		By direct differentiation, we see that the operator norm of $\nabla^2 \hat R$ scales like that of $AA^T$ which, 
		recalling classical properties of Wishart matricies~\cite{AGZ10}, scales like $M$ so that the empirical risk is $L$-smooth for $L$ that is of order $M$. A similar argument applies to the Lipschitz constant. On the other hand, one might rescale and consider instead  $\tilde{R}(x)=\frac{1}{M}\norm{Y-Ax}_2^2$. 
		The same reasoning, however, shows that $\|\tilde{R}\|_\infty\leq C$ with high probability. Thus the resulting invariant measure $\pi(dx)\propto e^{\tilde R(x)}dx$ is a uniformly bounded tilt of the uniform measure ($c dx \leq d\pi \leq C dx$ for some $C,c>0$ independent of $N$) and yields an exponentially small mass for any $\epsilon$-neighborhood of $v$ for $\eps<\pi/2$ by concentration of measure.
\end{rem}}

{
	\subsection{Online PCA}\label{subsec:online-PCA}
	From a statistical perspective, a classical application of online SGD to a matrix model is
	online PCA (sometimes called streaming PCA) whose analysis goes back to the work of~\cite{KRASULINA1969189} and ~\cite{OjaKarhunen}. The long history and vast literature on this problem makes it impossible to provide anywhere near a complete discussion of previous work, but we show that the problem fits into our classification. 
	Here one is given i.i.d.\
	samples $(Y^{\ell})$ from some distribution in an online fashion and the goal is to find
	the principal directions. Our results immediately apply in the commonly studied rank 1 setting where we assume that there is only
	one principal direction: namely, suppose that $(Y^{\ell})$ are i.i.d.\ centered
	random vectors in $\R^{N}$ with covariance $\E[YY^{T}]=S$, where
	\[
	S=I+\lambda v_0 v_0^{T}
	\]
	where $v$ is a fixed unit vector, $\lambda> 0$ is a fixed signal-to-noise
	ratio, and one iteratively optimizes the loss 
	\[
	\cL(x;Y)=-(x,Y Y^{T}x)
	\]
	over $\mathbb S^{N-1}$. We view this as a stochastic approximation to $- (x,\E[YY^{T}]x)$. 
	In the case that $Y^{\ell}\sim N(0,I+\lambda v_0 v_0^{T})$, the offline
	version of this problem immediately connects to the study of spiked
	Wishart matrices investigated in many works. 
	
	Observe that if we let $S=I+\lambda v_0v_0^{T}$, we have 
	\begin{equation}
	\Phi(x)=-(x,\E [Y Y^{T}]x)=-(x,Sx)=-\lambda(v_0,x)^{2}-1,\label{eq:online-pca-phi}
	\end{equation}
	since $x$ is a unit vector. Evidently this problem has information
	exponent $2$ so that the thresholds match those in the spiked matrix
	models of Section~\ref{subsec:tensor-pca}. In this setting, we have the following
	which holds for general $(Y^\ell)$ under mild moment assumptions. 
	\begin{prop}
		\label{prop:online-pca} Suppose that $\lambda>0$ is fixed and that $(Y^\ell)$ 
		are i.i.d.\ centered random vectors whose entries have bounded $10$-th moment: 
		$\sup_{i}\E [Y_{i}^{10}]<C$ for some $C>0$.
		Then Assumptions A and $B$ hold, and the model has information exponent $2$. 
	\end{prop}
	\noindent	The proof  of Proposition~\ref{prop:online-pca} is elementary and can be found in Appendix~\ref{app:streaming-pca}. 
}

\subsection{Spiked matrix and tensor models}\label{subsec:tensor-pca}
Another important class of examples are the spiked matrix and tensor models (also referred to as tensor PCA) \cite{johnstone2000distribution,Peche06,MR14}. 
Here we are given $M=\alpha N$ i.i.d.\ observations $(Y^\ell)$ of a rank 1 $p$-tensor on $\R^N$ corrupted by additive noise: 
\begin{equation}\label{eq:tensor-pca}
Y^\ell = J^\ell + \lambda v_0^{\tensor p}
\end{equation}
where $(J^{\ell})$ are i.i.d.\ copies of $p$-tensors with i.i.d.\ entries of mean zero and variance one, and $\lambda$ is a signal-to-noise parameter.  The goal in these problems is to infer the vector $v_0\in \mathbb R^N : \|v_0\|=1$. 
A standard approach to inferring $v_0$ is by optimizing the
following $\ell^2$ loss: 
\begin{equation}\label{eq:tensor-pca-loss}
\mathcal{L}(x; Y) : = \big(Y,x^{\tensor p}\big)=  (J,x^{\otimes p})+\lambda\big(x\cdot v_0\big)^{p}.
\end{equation}
When $J$ is Gaussian, and we restrict to $\mathbb{S}^{N-1}$, this is simply maximum likelihood estimation. When
$J$ is non-Gaussian, this can be thought of as computing the
best rank 1 approximation to $J$.

We obtain the following result regarding the information exponent of these models whose proof is deferred to Appendix~\ref{app:tensor-pca}.

\begin{prop}\label{prop:tensor-pca-information-exponent}
	Consider the spiked tensor model where $\lambda = 1$ and $J$ is an i.i.d.\ $p$-tensor with loss \eqref{eq:tensor-pca-loss}.  
	Suppose that entries of $J$ have mean zero and finite $6$th moment.
	Then Assumptions A and B hold and the population loss has information exponent $p$.
\end{prop}

From this we see that the spiked matrxi model falls in our quasi-linear class,
whereas the spiked tensor model is in the polynomial class.
We now discuss the relation between our sample complexity thresholds for online SGD and those found for other algorithms in preceding work.  
For the sake of that comparison, we recall here that
the information theoretic threshold for this estimation problem, with the above scaling is at $\lambda \sqrt \alpha$  of order one~\cite{Peche06,mont15,BMPW16,LMLKZ17,BMMN17,BBCR18,JLM20}. 

\subsubsection{Spiked tensor models}
The tensor case has recently seen a surge of interest as it is expected to be an example of an inference problem which has a diverging statistical-to-algorithmic gap.
Strong evidence for this comes from the analysis of spectral and sum-of-squares-type methods where sharp thresholds of the form $\lambda\sqrt{\alpha}\sim N^{\frac{k-2}{4}}$ have been obtained for efficient estimation \cite{MR14,kim2017community,hopkins2015tensor,Kikuchi-tensor-PCA,hopkins2016fast}. 

On the other hand, it was conjectured \cite{MR14} that power iteration and approximate message passing should have a threshold of $\lambda\sqrt\alpha = O(N^{\frac{k-2}{2}})$. It has been speculated that the latter threshold may be common to first-order methods without any global or spectral initialization step.

In~\cite{BGJ18a}, efficient recovery was proved for $\lambda \sqrt \alpha$ growing faster than $N^{(k-2)/2}$ for gradient descent and Langevin dynamics, using a differential inequalities based argument originating from~\cite{BGJ18a} and evidence  was provided for hardness when it is smaller.  (See also~\cite{passed-and-spurious,big-bad-minima,marvels-pitfalls,iron-out} for related, non-rigorous analyses.)   
The threshold we find of $\alpha \sim N^{k-2}$ in Theorem~\ref{thm:main-1}, exactly matches this, and shows that the
online SGD attains the (conjecturally) optimal thresholds for first order methods. {Note, however, that the above works required the random tensor $J$ to be Gaussian, whereas Proposition~\ref{prop:tensor-pca-information-exponent} covers non-Gaussian $J$ as well.}

\begin{rem}\label{rem:spiked-matrix}
	Similarly to the online PCA example, when $p=2$ our results imply a threshold for online SGD of at least $\alpha \sim \log N$ and at most $\alpha \sim (\log N)^2$ for spiked matrix models of the form~\eqref{eq:tensor-pca}. The reader may note that one expects to be able to solve spiked matrix estimation tasks with $\alpha = O(1)$, e.g., in the offline setting using gradient descent. 
	The $\log N$ factor in our results is due to the on-line nature of this setting and should be compared to the run-time of 
	gradient descent in the offline setting with a random start, which will indeed take $\log N$ time just to weakly recover. 
\end{rem}

\subsubsection{Spiked tensor-tensor and matrix-tensor models: min-stability of information exponents}\label{subsec:tensor-tensor-models}
Recently, there have been several results regarding the so-called spiked matrix-tensor and spiked
tensor-tensor models, where one is given a pair of tensors of the
form $$Y=\left(J+\lambda v_{0}^{\tensor p},\tilde{J}+\lambda v_{0}^{\tensor k}\right)\,,$$
where $J$ and $\tilde{J}$ are independent $p$ and $k$-tensors respectively. Here one is interested
in inferring $v_{0}$ via the sum of the losses. Evidently these problems
will have information exponent $\min\{p,k\}$. 

In the case $p>k=2$, scaling limits (as $N\to\infty $ and then  $m_0\downarrow 0^+$) of approximate message passing through state evolution equations and gradient descent through the Crisanti--Horner--Sommers--Cugliandolo--Kurchan equations have been studied in~\cite{big-bad-minima,passed-and-spurious, marvels-pitfalls}. 
In our results, we investigate the regime $m_0\sim N^{-1/2}$ corresponding to a random start, and find that online SGD recovers for $\alpha \gtrsim (\log N)^2 $ and fails for $\alpha = o(\log N)$. 
We expect that $\alpha = \Theta(\log N)$ is in fact the true threshold, since (offline) gradient decent requires $N \log N$
time to optimize the population loss from a random start.

\subsection{Two-component mixture of Gaussians: an easy-to-critical transition}\label{subsec:mixture-of-Gaussians}

Consider the case of maximum likelihood estimation for a mixture of
Gaussians with two components. We assume for simplicity that
the variances are identical and that the mixture weights are known. We consider 
the ``spherical'' case and assume that the clusters
are antipodal, i.e., 
\begin{equation}\label{eq:Gaussian-mixture}
Y\sim p \cN(v_0,Id)+(1-p)\cN(-v_0,Id)\qquad \mbox{for}\qquad v_0\in \mathbb R^N: \norm{v_0}=1\,.
\end{equation}
Without loss of generality, take $v_0=e_{1}.$ The log-likelihood
in this case is of the form 
\[
\tilde{f}(x,Y)= \log\Bigg(p\cdot \exp\Big(-\frac{1}{2}\norm{Y-x}^{2}\Big)+(1-p)\cdot\exp\Big(-\frac{1}{2}\norm{Y+x}^2\Big)\Bigg).
\]
To place this in to our framework, it will be helpful to re-parameterize the mixture weights as $p \propto e^{h}$  for some $h\in\R$.
In this setting, we see that maximum likelihood estimation is equivalent to minimizing the loss 
\[
\cL(x;Y) = -\log\cosh(Y\cdot x + h).
\]
\begin{prop}\label{prop:mixture-model}
	Consider a two-component Gaussian mixture model as in \eqref{eq:Gaussian-mixture}.
	If we take as loss the (negated) log-likelihood, then Assumptions A and B hold.
	Furthermore, if $p\neq 1/2$, it has information exponent $1$, and if $p= 1/2$, it has information exponent $2$. 
\end{prop}

The proof of Proposition~\ref{prop:mixture-model} is deferred to Appendix~\ref{app:mixture-model}.

While there is a huge literature in the Gaussian mixture model setting, we mention a few recent results related to our work. 
As a consequence of these results we obtain an $O(N \log^2 N)$ sample complexity upper bound for online SGD. 
It is known that, from the perspective of learning the density in TV distance, this is optimal, see \cite{KMV10,near-optimal-NIPS,nearly-tight-NIPS}. {With different assumptions on the sample complexity,~\cite{MBM18} used a critical point analysis to understand the behavior of true gradient descent.}

\section{Analysis of two stages of performance}
We now turn the proofs of our results.  We begin by stating our main technical results 
for the search and descent stages of performance of SGD, 
together giving Theorem~\ref{thm:main-1}. These two results show that dynamics attains order one correlation and exits the search phase on a timescale of $\tilde O(\alpha_c(N,k)N )$ and, from there, well-approximates the population dynamics until it quickly attains $1-o(1)$ correlation.  
Without loss of generality
and for notational convenience we will take $\theta_N =  e_1$ for the remainder of this paper, where $e_1$ denotes the usual
Euclidean basis vector.

\subsection{Search phase}
Our main results for the search phase shows that the timescale for attaining weak recovery when started at any point with $m_0= \Omega( N^{-1/2})$ is of order $\tilde O(\alpha_c(N,k)N)$. 	For this weak recovery result, we actually do not need the full strength of Assumption A. To that end, for $0<\varrho\leq1$, let us introduce \textbf{Assumption $\text{A}_\varrho$}  that $\phi$ is differentiable, and $\phi'$ is strictly negative on $(0,\varrho)$. Evidently Assumption A implies Assumption $A_\varrho$ for every $\varrho$.

\begin{thm}
	\label{thm:main-theorem} 
	Suppose there exists $\varrho>0$ such that Assumptions $\text{A}_{\varrho}$ and B hold and that the population loss has information exponent $k$. Let $(\alpha_N,\delta_N)$ be in Theorem~\ref{thm:main-1}. 
	Then there exists $\eta>0$
	such that if $X_{t}  = X_{t}^{N,\delta}$ is the online
	SGD with step size $\delta$, we have for every $\gamma>0$, 
	\begin{align*}
	\lim_{N\to\infty}\,\inf_{x_{0}:m(x_{0})\geq\gamma/\sqrt{N}}\,\mathbb{P}_{x_{0}}\Big(\tau_\eta^+ < \alpha N \Big)= & 1\,.
	\end{align*}
	where, we recall, $\tau_\eta^+$ is the stopping time $\inf\{t: m_t > \eta\}$. 
\end{thm}

Before turning to the corresponding refutation theorem, let us pause to comment on the role of initialization in this result.
We note here that this result is uniform over any initial data with $\{m(x)\geq \gamma/\sqrt N\}$ for $\gamma>0$. For $m(x)$ on this scale the sample-wise error for the loss, $H_N$, dominates the population loss substantially. In particular,
both the sample and empirical losses in that region are typically highly non-convex.
This scaling, however, is the natural scaling for initializations in high-dimensional problems as the uniform measure on the upper half-sphere satisfies $m_0$ of order $N^{-1/2}$ with probability $1-o(1)$~\cite{LedouxTalagrand}: in this manner, the assumption on the initialization of Theorem~\ref{thm:main-theorem} is weaker than that of Theorem~\ref{thm:main-1}. For a discussion of initializations having $m_0= o(N^{-1/2})$ c.f.\ Remark~\ref{rem:starting-in-upper-half-sphere}, where we noted that the timescales to recovery would depend on a lower bound on the variance of the directional-error martingale, corresponding to an additional lower-bound assumption in~\eqref{eq:condition-1} of Assumption A. 

In an analogous manner, the initialization in Theorem~\ref{thm:refutation-1} can be boosted to be uniform over initializations having $m_0 = O(N^{-1/2})$; indeed this is the version we will prove.

\subsection{Descent phase}
In the search phase, the recovery follows by showing that $m_t$ obeys a differential inequality comparable to that satisfied by $\bar m_t$, the correlation of the population dynamics: SGD on the population loss, as defined below. This shows that whereas $m_t - \bar m_t$ may be quite large (the directional-error martingale is on the same scale as $m_t$), the timescale of the hitting time $\tau_\eta^+$ is the same as that of the population dynamics. In the descent phase, low-dimensional intuition applies and we establish that $\bar m_t$ is indeed a good approximation to the trajectory of $m_t$, leading to consistent estimation in a further linear time. 

To be more precise, let $\overline X_t$ be the \emph{population dynamics}, i.e., given by the following procedure
\begin{align*}
\overline X_t =  \frac{1}{\|\overline X_{t-1} -  \frac{\delta}{ N} \nabla \Phi(\overline X_{t-1})\|} \Big(\overline X_{t-1} -\frac{\delta}{ N} \nabla \Phi(\overline X_{t-1})\Big)\,.
\end{align*}
Since $\Phi(x) = \phi(m)$, the population dynamics can be reduced to its 1D projection onto the correlation variable.   Under Assumption A, we have that for every $\delta = O(1)$, and initializations $\overline{X}_{0}$ with $\liminf_N m(\overline{X}_0)>0$, 
$$\lim_{T\to\infty}\liminf_{N\to\infty} m(\overline{X}_{T\delta^{-1} N})\to 1\,.$$
We also note that if Assumption A does not  hold, then the above would \emph{not} hold.

\begin{thm}\label{thm:lln}
	Suppose that Assumption A and B hold. Fix any $\eta>0$ and
	let $X_0 = \overline{X}_{0}$ be any point such that $m(\overline{X}_{0})=\eta$.
	For $\alpha_N = \omega(1)$, and every $\alpha \delta^2 =o(1)$, we have 
	\begin{align}\label{eq:lln-fixed-start}
	\sup_{\ell \le M}\big|m(X_{\ell}) - m(\overline{X}_{\ell})\big|\to0 & \qquad\mbox{in }\mathbb P\text{-prob}.
	\end{align}
\end{thm}

In this way, upon attaining non-trivial correlation with the ground truth, and reaching the descent phase, the SGD behaves as it would in a low-dimensional trust region and the usual ODE method applies. Namely, the high-dimensionality of the landscape is no longer relevant. Moreover, the information exponent no longer plays a central role, and the descent always takes linear time.  

\subsection{Linear sample complexity when $k=1$}Problems in the $k=1$ class behave the same in their search and descent phases, and are therefore easy in the sense that their entire trajectory satisfies the law of large numbers of Theorem~\ref{thm:lln}. As a result, we are able to sharpen the reuslts above in the $k=1$ case to analyze the situations where $\alpha_N$ is of order one (as opposed to diverging/going to zero arbitrarily slowly as was assumed in Theorem~\ref{thm:main-1}), at the cost of only attaining weak recovery. As mentioned in Remark~\ref{rem:k=1}, in the linear sample complexity regime, there are problems for which weak recovery is the best that is information theoretically possible. 

\begin{thm}\label{thm:k=1}
	Suppose that Assumptions A and B hold and that the population loss has information exponent $k=1$. For every $\epsilon>0$, the following holds: 
	\begin{enumerate}[(a)]
		\item There exists $\eta>0$ and $\alpha_0>0$, such that every $\alpha>\alpha_0$, there is a choice of $\delta>0$ such that $X_t = X_t^{N,\delta}$ initialized from the uniform measure on the upper half-sphere $\mu_N$ satisfies
		\begin{align*}
		\liminf_{N\to\infty}\mathbb P_{\mu_N} \big(m(X_M)\geq 1-\eta\big) \geq  1-\epsilon\,;
		\end{align*}
		\item For every $\eta>0$ there exists $\alpha_{0}'>0$ such that for every $\alpha<\alpha_{0}'$, and every $\delta = O(1)$,
		\begin{align*}
		\limsup_{N\to\infty} \mathbb P_{\mu_N} \big( \max_{t\leq M} m(X_t)>\eta\big) \leq \epsilon\,.
		\end{align*}
	\end{enumerate}
\end{thm}

\noindent Theorem~\ref{thm:k=1} will be proved in conjunction with the proofs of Theorems~\ref{thm:main-1}--\ref{thm:refutation-1}.

\section{A Difference inequality for Online SGD}

The key technical result underlying our arguments is to show that
with high probability, the value of $m_t = m(X_t)$ is a super-solution
to (a constant fraction of) the integral equation satisfied by the underlying population dynamics. To state this result,
we introduce the following notation.

Throughout this section, we view $X_{0}=x_{0}$ as a fixed initial
data point and recall the notation $\mathbb P_{x_0}$ emphasizing the dependence of the output of the algorithm on the initial data. We prove all the results in the general setting where the population loss $\Phi_N = \phi_N(m_N(x))$ may be such that $\phi$ also depends on $N$ (see Remark~\ref{rem:phi-N-dependence} for the relevant modifications to Assumption A and the information exponent). From now on we suppress the $N$ dependence in these notations whenever clear from context.

Recall the definition of $m(x)$ from~\eqref{eq:Phi-form}; we will use the following notations for spherical caps: 
\begin{equation}
E_{\eta}=\left\{ x\in\mathbb{S}^{N-1}:m(x)\geq\eta\right\} \,.\label{eq:E-set-def}
\end{equation}
Our results in this section focus on arbitrary $X_{0}=x_0$ in
$E_{\gamma/\sqrt{N}}\setminus E_{\eta}$, for some $\eta$ sufficiently
small but positive. If instead $m(X_{0})>\eta$, then Theorem~\ref{thm:main-1} follows
immediately from Theorem~\ref{thm:lln}. 

For every $\theta$, define the hitting times 
\begin{align*}
\tau_{\theta}^{+}:=\min\{t\geq0:m(X_{t})\geq\theta\}\,,\qquad\mbox{and}\qquad & \tau_{\theta}^{-}:=\min\{t\geq0:m(X_{t})\leq\theta\}\,.
\end{align*}
Fix $\iota>0$ given by \eqref{eq:condition-2}
and define 
\begin{equation}
\bar{L}:=\sup_{x}\E\big[\abs{\frac{1}{\sqrt N}\nabla H(x)}^{4+\iota}\big]\vee\sup_{x}\E\big[\abs{\frac{1}{\sqrt N}\nabla H(x)}^{2}\big]\vee 1\,.\label{eq:bar-l-def}
\end{equation}
{Furthermore, let $a_k =  c\cdot k$  and $a_{k+1}=C\cdot (k+1)$, where $C,c$ are as in Definition~\ref{def:inf-exponent}.} 
Our first goal is to obtain a difference inequality for the evolution
of $m_{t}:=m_N(X_{t})$ that holds for all $t\leq\tau_{\gamma/(2\sqrt{N})}^{-}\wedge\tau_{\eta}^{+}\wedge \bar t$
where $\bar t\leq M$ is some guaranteed recovery time for the algorithm.

\begin{prop}
	\label{prop:main-difference-inequalities} Suppose that Assumptions~$\text{A}_{\varrho}$ and B hold and that the population loss has information
	exponent $k$. Let $D= D_N$ and $\varepsilon= \varepsilon_N = O(1)$, and suppose $\alpha= \alpha_N$ is of at most polynomial growth in $N$, and $\delta = \delta_N$ is such that $\  \delta^2\leq \varepsilon$ and for some $K>0$, 
	\begin{align}\label{eq:delta-parameter}
	\delta \leq \bar \delta _N(k) :=\begin{cases} \frac{a_1}{4K\bar L} & k=1 \\ \frac{a_{k}\gamma^{k-2}}{K\bar{L}N^{\frac{k-2}{2}}\log N} & k\geq 2\end{cases}\,.
	\end{align}
	Then for every $\gamma>0$ and every $T\leq M:= \alpha N$ satisfying 
	\begin{equation}
	T\leq\frac{N\gamma^{2}}{D^{2}\delta^{2}}=:\bar{t}\,,\label{eq:tbar}
	\end{equation}
	online SGD with step-size $\delta$ satisfies the following as $N\to\infty$ for some $\eta>0$, uniformly over the choice of $D, \varepsilon , K$: 
	\begin{enumerate}
		\item If $k=1$, there exists a constant $C=C(C_1, a_1, a_2)>0$ such that
		\begin{align}
		\inf_{x_{0}\in E_{\gamma/\sqrt{N}}}\inf_{t\leq T}\mathbb{P}_{x_{0}}\left(m_{t}\geq\frac{m_{0}}{2}+\frac{\delta a_{k}}{8N}t\quad \forall t\leq\tau_{\gamma/(2\sqrt{N})}^{-}\wedge\tau_{\eta}^{+}\right)\geq1- C \Big(\frac{1}{K} - \frac{1}{D^2}\Big) -o(1)& ,\label{eq:k=1-inequality}
		\end{align}
		\item If $k\geq2$, there exists a constant $C(C_1, a_k, a_{k+1})>0$ such that
		\begin{equation}
		\inf_{x_{0}\in E_{\gamma/\sqrt{N}}}\mathbb{P}_{x_{0}}\Big(m_{t}\geq\frac{m_{0}}{2}+\frac{\delta a_{k}}{8N}\sum_{j=0}^{t-1}m_{j}^{k-1}\quad\forall t\leq\tau_{\gamma/(2\sqrt{N})}^{-}\wedge\tau_{\eta}^{+}\wedge T\Big) \geq 1 - \frac{C}{D^2}-o(1)\,.\label{eq:k>=2-inequality}
		\end{equation}
	\end{enumerate}
\end{prop}

The proof of Proposition~\ref{prop:main-difference-inequalities} will
follow in three stages. In Section \ref{subsec:controlling-radial-effects},
we split the evolution of $m(X_{t})$ in three parts: the drift induced
by the population loss, a martingale induced by the gradient of the
sample-wise error in the direction of $e_{1}$, and a second order
(in $\delta$) effect caused by the non-linear projection step in
the algorithm. In Section~\ref{subsec:second-order-correction}--\ref{subsec:controlling-the-drift}, we show that the second order effect is bounded
by the drift of the population loss. In Section~\ref{subsec:directional-error-martingale} we control the martingale;
while its contribution dominates the drift on short time scales, its
total contribution by time $T\leq M$ is smaller than the initial
bias $\gamma/\sqrt{N}$. We end this section with the proof of \prettyref{prop:main-difference-inequalities}.

\subsection{Controlling radial effects \label{subsec:controlling-radial-effects}}

Let us begin by controlling the effect of the projection in \prettyref{eq:Algorithm-def} and obtain a difference equation for $m_t$.
By the chain rule, $\nabla\Phi(x)=\phi'(m(x))\nabla m(x)$, where
\begin{align*}
\nabla m(x)=  e_{1}-(x\cdot e_{1}) x\,.
\end{align*}
Since $\|m\|_{L^{\infty}(\mathbb{S}^{N-1})}\leq1$, and $\phi'$ is
continuous, there exists $A>0$ , depending only on $a_{k},a_{k+1}$
(and not on $N$), such that 
\[
\sup_{x}\abs{\nabla\Phi(x)}^{2}=\sup_{x}|\phi'(m(x))\nabla m(x)|{}^{2}\leq {A}\,.
\]

By Jensen's inequality and~\eqref{eq:condition-2}, $\bar{L}=O(1)$. Let 
\begin{equation} 
L_{t}=\abs{\frac{1}{\sqrt N}\nabla H^{t}(X_{t-1})}^{2}\,,\label{eq:L_t-def}
\end{equation}
and observe that 
\[
|{\frac{1}{\sqrt{N}}\nabla\cL(X_{t-1};Y^{t})}|^{2} \leq2\left(\frac{A}{N}+L_{t}\right).
\]
Recalling~\eqref{eq:Algorithm-def}, let $r_{t}=\abs{\tilde{X}_{t}}$, and note that by orthogonality
of $\nabla\mathcal{L}(X_{t-1};Y^{t})$ and $X_{t-1}$, 
\[
r_{t}\leq\sqrt{1+(\delta|\nabla\mathcal{L}(X_{t-1};Y^{t})|/N)^2}.
\]
Since $1\leq\sqrt{1+u}\leq1+u$ for every $u\geq0,$ we obtain
\begin{equation}\label{eq:r-bound}
1\leq r_{t}\leq1+\delta^2\left(\frac{A}{N^2}+\frac{L_t}{N}\right).
\end{equation}
By definition of $X_{t}$ , we then see that for every $t\geq1$,
\begin{align}\label{eq:m-t-finite-difference}
m_{t}=\frac{\tilde{X}_{t}\cdot e_{1}}{r_{t}}=\frac{1}{r_{t}}\left(m_{t-1}{-}\frac{\delta}{{N}}\nabla\Phi(X_{t-1})\cdot e_{1}{-}\frac{\delta}{{N}}\nabla H^{t}(X_{t-1})\cdot e_{1}\right).
\end{align}
Since for $u\geq0,$ we have $\abs{\frac{1}{1+u}-1}\leq\abs{u}$,
we may combine these estimates to obtain
\begin{align}
m_{t} & \geq m_{t-1}{-}\frac{\delta}{{N}}\nabla\Phi(X_{t-1})\cdot e_{1}{-}\frac{\delta}{{N}}\nabla H^{t}(X_{t-1})\cdot e_{1}\nonumber \\
& \quad-\delta^{2}(\frac{A}{N^2}+ \frac{L_t}{N})|m_{t-1}|-\delta^{3}(\frac{A}{N^2}+ \frac{L_t}{N})\Big(\Big|\frac{\nabla\Phi(X_{t-1})\cdot e_{1}}{{N}}\Big|+\Big|\frac{\nabla H^{t}(X_{t-1})\cdot e_{1}}{{N}}\Big|\Big).\label{eq:zeroth-bound-m}
\end{align}
To better control the terms that are second order in $\delta$, let
us introduce a truncation of $L_{t}$. Fix a truncation value $\hat{L}>0$,
possibly depending on $N$, to be chosen later. We can rewrite~\eqref{eq:zeroth-bound-m} as 
\begin{align}
m_{t} & \geq m_{t-1}{-}\frac{\delta}{{N}}\nabla\Phi(X_{t-1})\cdot e_{1}{-}\frac{\delta}{{N}}\nabla H^{t}(X_{t-1})\cdot e_{1}-\delta^{2}\Big(\frac{L_{t}\mathbf{1}_{\{L_{t}<\hat{L}\}}}{N}\Big)|m_{t-1}|\label{eq:first-bound-m}\\
& \quad-\delta^{2}\Big(\frac{A}{N^{2}}+\frac{L_{t}\mathbf{1}_{\{L_{t}\geq\hat{L}\}}}{N}\Big)|m_{t-1}|-\delta^{3}(\frac{A}{N^{2}}+\frac{L_{t}}{N})\Big(\Big|\frac{\nabla\Phi(X_{t-1})\cdot e_{1}}{{N}}\Big|+\Big|\frac{\nabla H^{t}(X_{t-1})\cdot e_{1}}{{N}}\Big|\Big).\nonumber 
\end{align}
We now turn to controlling the second order in $\delta$ correction
terms, which we will see have to be treated separately for the one
on the first line above, and those on the second line above. 

\subsection{Bounding the higher order corrections\label{subsec:second-order-correction}}
We begin with a priori bounds on the higher order terms in~\eqref{eq:first-bound-m}. 
Observe that for every $\eta<\frac{1}{2}$, for every $x\in\{x:m(x)\in[0,\eta]\}$,
\[
\nabla m(x)\cdot e_{1}=(e_{1}-(x\cdot e_{1}) x)\cdot e_{1}=1-m(x)^{2}\geq1-\eta^{2}\geq\frac{1}{2}.
\]
Then there exists $\eta_{0}(\varrho,a_{k},a_{k+1})>0$ such that for all
$\eta<\eta_{0}$ , for all $x\in\{x:m(x)\in[0,\eta]\}$, 
\begin{align}
\frac{1}{4}a_{k}(m(x))^{k-1}\leq{-}\nabla\Phi(x)\cdot e_{1}\leq & \frac{3}{2}a_{k}(m(x))^{k-1}\,.\label{eq:D1phi-bounds}
\end{align}
With this in hand, the next lemma gives a pairing of the second order
term on the second line of \eqref{eq:first-bound-m}, which we bound
subsequently by comparison to the initial $m_0$, and the second order term on the first line of \eqref{eq:first-bound-m}
which we bound by comparison to the first order (in $\delta$) drift
term. 
\begin{lem}
	Let $\eta,\gamma>0$ with $\eta<1/2$. For every $K>0$, and every $\delta \leq \bar \delta_N(k)$ as in~\eqref{eq:delta-parameter}: 
	\begin{enumerate}
		\item If $k=1$, for every $x\in\{x:m(x)\in[0,\eta]\}$, 
		\begin{equation}
		\frac{\delta^{2}}{N}\abs{m(x)}\leq\frac{a_{k}}{4K\bar{L}}\frac{\delta}{N}\,,\label{eq:first-second-order-term-k=1}
		\end{equation}
		\item If $k\geq2$, for every $x\in\{x:m(x)\in[\gamma/(2\sqrt{N}),\eta]\}$, 
	\end{enumerate}
	\begin{equation}
	\frac{\delta^{2}}{N}\abs{m(x)}\leq\frac{\delta}{N}\frac{a_{k}\abs{m(x)}^{k-1}}{K\bar{L}\log N}.\label{eq:first-second-order-term}
	\end{equation}
\end{lem}

\begin{proof}
	First let $k=1$ and suppose $x\in\{x:m(x)\in[0,\eta]\}$. For every
	$K>0$ and every $N$ , if we take $\delta\leq a_{k}/(4K\bar{L})$ then
	we have 
	\[
	{\delta^{2}}\abs{m(x)}\leq \delta\frac{a_{k}}{4K\bar{L}} \,.
	\]
	Similarly, if $k\geq2,$ we see that if $x\in\{x:m(x)\in[\gamma/(2\sqrt{N}),\eta]$,
	we have by~\eqref{eq:delta-parameter},
	\[
	\delta^{2}\abs{m(x)} \leq\frac{a_{k}}{K\bar{L}\log N}\delta\abs{m(x)}\Big(\frac{\gamma}{\sqrt{N}}\Big)^{k-2}\leq\frac{a_{k}}{K\bar{L}\log N}\cdot\delta\abs{m(x)}^{k-1}. \qedhere
	\]
\end{proof}
\begin{lem}
	Suppose that $\alpha\delta^{2}\leq\varepsilon$ for some $\varepsilon>0$,
	let $\bar{L}$ be as in \prettyref{eq:bar-l-def}. Then there is a
	constant $C=C(\bar{L},A,C_{1},C_{2})>0$ such that the following hold uniformly over $x_0\in \mathbb S^{N-1}$ 
	\begin{equation}
	\mathbb{P}_{x_{0}}\Bigg(\sup_{t\leq M}\delta^{3}\sum_{j=0}^{t-1}\Big(\frac{A}{N^{2}}+\frac{L_{j+1}}{N}\Big)\Big(\Big|\frac{\nabla\Phi(X_{j})\cdot e_{1}}{{N}}\Big|+\Big|\frac{\nabla H^{j+1}(X_{j})\cdot e_{1}}{{N}}\Big|\Big)>\frac{\gamma}{10\sqrt{N}}\Bigg)\leq\frac{C \varepsilon}{\gamma N^{3/2}}\,.\label{eq:second-second-order-term}
	\end{equation}
	\begin{align}
	\mathbb{P}_{x_{0}}\Big(\sup_{t\leq M}\delta^{2}\sum_{j=0}^{t-1}\Big(\frac{A}{N^{2}}+\frac{L_{j+1}\mathbf{1}_{\{L_{j+1}\geq\hat{L}\}}}{N}\Big)|m_{j}|>\frac{\gamma}{10\sqrt{N}}\Big)\leq & \frac{C\varepsilon\sqrt{N}}{\hat{L}^{1+\frac{\iota}{2}}\gamma}\,.\label{eq:truncation-term}
	\end{align}
\end{lem}

\begin{proof}
	Both bounds follow from Markov's inequality. Since the summands are
	positive, the suprema over $t\leq M$ are attained at $t=M$, so it
	suffices to consider that case. Fix $x_{0}\in\mathbb{S}^{N-1}$. The bounds
	will be seen to be uniform in this choice.
	We begin with~\eqref{eq:second-second-order-term}; for every $\lambda>0$, 
	\begin{align*}
	\mathbb{P}_{x_{0}}\left(\sum_{j=0}^{M-1}\delta^{3} \right.&\left.\Big(\frac{A}{N^{2}}+\frac{L_{j+1}}{N}\Big) \Big(\Big|\frac{\nabla\Phi(X_{j})\cdot e_{1}}{{N}}\Big|+\Big|\frac{\nabla H^{j+1}(X_{j})\cdot e_{1}}{{N}}\Big|\Big)>\lambda\right)\\
	\leq & \frac{M\delta^{3}}{\lambda N^2}\sup_{x}\mathbb{E}\Bigg[\Big(\frac{|\nabla H(x)|^{2}}{N}+\frac{A}{N}\Big)\cdot\Big(\big|{\nabla\Phi(x)\cdot e_{1}}\big|+\big|{\nabla H(x)\cdot e_{1}}\big|\Big)\Bigg]\\
	\leq & \frac{2\alpha\delta^{3}}{\lambda N^2}\sqrt{\sup_{x}\mathbb{E}\Big[\Big|\frac{\nabla H(x)}{\sqrt N}\Big|^{4}\Big]+\frac{A^{2}}{N^{2}}}\sqrt{\sup_{x}\big|\nabla\Phi(x)\cdot e_{1}\big|^{2}+\sup_{x}\mathbb{E}\big[\big|\nabla H(x)\cdot e_{1}\big|^{2}\big]} \\
	\leq & \frac{2\alpha\delta^{3}}{\lambda N^2}\sqrt{\Big(\bar{L}+\frac{A^{2}}{N^2}\Big)\Big(A+C_{1}\Big)}\,,
	\end{align*}
	where the second inequality follows by Cauchy-Schwarz, and the third
	from natural scaling of $(\mathbb P, \mathcal L)$ and \eqref{eq:D1phi-bounds}. If we take $\lambda=\gamma/(10\sqrt{N})$,
	and use the fact that $\alpha\delta^{2}\leq\varepsilon$, we obtain
	the first bound.
	
	Now, let us turn to the second bound. Again, as the summands are positive,
	the supremum is attained at $t=M$, so that
	\begin{align*}
	\mathbb{P}_{x_{0}}\Big(\sup_{t\leq M}\sum_{j=0}^{t-1}\delta^{2}\Big(\frac{A}{N^{2}}+\frac{L_{j+1}\mathbf{1}_{\{L_{j+1}\geq\hat{L}\}}}{N}\Big)|m_{j}|>\lambda\Big)& \leq  \mathbb{P}_{x_{0}}\bigg(\delta^{2}\sum_{j=0}^{M-1}\Big(\frac{A}{N^{2}}+\frac{L_{j+1}\mathbf{1}_{\{L_{j+1}>\hat{L}\}}}{N}\Big)>\lambda\bigg)\\
	& \leq \prob\bigg(\delta^{2}\sum_{j=0}^{M-1}\frac{L_{j+1}\indicator{\{L_{j+1}>\hat{L}\}}}{N}>\lambda-\frac{A\cdot\varepsilon}{N}\bigg),
	\end{align*}
	where we used here that since $\alpha\delta^{2}\leq\varepsilon$,
	we have $\delta^{2}\sum_{j=0}^{M-1}\frac{A}{N^{2}}\leq\frac{A\cdot\varepsilon}{N}$.
	By Markov's inequality, 
	\begin{align*}
	\mathbb{P}_{x_{0}}\Big(\sum_{j=0}^{M-1}L_{j+1}\mathbf{1}_{\{L_{j+1}>\hat{L}\}}>\Lambda\Big) & \leq\frac{M}{\Lambda}\Big(\sup_{j}\mathbb{E}_{x_{0}}[L_{j+1}\mathbf{1}_{\{L_{j+1}>\hat{L}\}}]\Big)\\
	& \leq\frac{M}{\Lambda}\sqrt{\sup_{x}\mathbb{E}[|\frac{1}{\sqrt{N}}\nabla H(x)|^{4}]\cdot\sup_{x}\mathbb{P}(|\frac{1}{\sqrt{N}}\nabla H(x)|^{2}>\hat{L})}\\
	& \leq\frac{\alpha N}{\Lambda}\frac{\bar{L}}{\hat{L}^{1+\iota/2}}\,,
	\end{align*}
	where in the last line we again used Markov's inequality along with
	the definition of $\bar{L}$ from \eqref{eq:bar-l-def}. Choosing
	$\Lambda=\frac{N}{\delta^{2}}(\lambda-\frac{A\cdot\varepsilon}{N})$,
	with $\lambda=\gamma/(10\sqrt{N)}$, we get that for $N$ large, the
	left-hand side in \eqref{eq:truncation-term} is bounded by
	\[
	\frac{C\alpha\delta^{2}\sqrt{N}}{\hat{L}^{1+\iota/2}\gamma}\leq\frac{C\varepsilon\sqrt{N}}{\hat{L}^{1+\iota/2}\gamma}\,,
	\]
	for some $C(A,\bar{L})>0$ as
	desired. 
\end{proof}
We now sum \eqref{eq:first-second-order-term} (or \eqref{eq:first-second-order-term-k=1}
if $k=1$) over $t\geq1$, and combine this with \prettyref{eq:D1phi-bounds},
\eqref{eq:second-second-order-term}, and \prettyref{eq:truncation-term}
with the choice of $\hat{L}=N^{\frac{1}{2}-\frac{\iota}{4}}$. We
see that for every $\gamma>0$, $\eta<\eta_0$ and $K>0$, if $k=1$, 
\begin{equation}
\lim_{N\to\infty}\inf_{x_{0}}\mathbb{P}_{x_{0}}\Bigg(\begin{array}{ll}\!\!m_{t}\geq& \!\frac{4}{5}m_{0}+\sum_{j=0}^{t-1}\frac{\delta a_{k-1}}{4N}\Big(1-\frac{L_{j+1}\mathbf{1}_{\{L_{j+1}<\hat{L}\}}}{K\bar{L}}\Big) \\ 
& \quad \quad {-}\frac{\delta}{{N}}\sum_{j=0}^{t-1}\nabla H^{j+1}(X_{j})\cdot e_{1}\end{array} \quad \forall t<\tau_{0}^{-}\wedge\tau_{\eta}^{+}\Bigg)=1 \,,\label{eq:second-order-is-controlled-k=1}
\end{equation}
and if $k\geq2$, 
\begin{align}
\lim_{N\to\infty}\inf_{x_{0}}\mathbb{P}_{x_{0}}\Bigg(\begin{array}{ll}\!\! m_{t}\geq& \! \!\! \frac{4}{5}m_{0}+\sum_{j=0}^{t-1}\frac{\delta a_{k-1}|m_{j}|^{k-1}}{4N}\Big(1-\frac{L_{j+1}\mathbf{1}_{\{L_{j+1}<\hat{L}\}}}{K\bar{L}\log N}\Big)\\
&  \quad \quad \!\!{-}\frac{\delta}{{N}}\sum_{j=0}^{t-1}\nabla H^{j+1}(X_{j})\cdot e_{1}\end{array}\! \!\!\quad \forall t<\tau_{{\gamma}/{(2\sqrt{N})}}^{-}\wedge\tau_{\eta}^{+}\Bigg)= & 1\,.\label{eq:second-order-is-controlled-k=2}
\end{align}
(Here we used the fact that $t<\tau_{0}^{-}$ to replace $m_{j}$
with $|m_{j}|$ on the right hand sides). 
Notice that these limits hold uniformly over the choices of $D$ and $\varepsilon=O(1)$. 

\subsection{Controlling the drift\label{subsec:controlling-the-drift}}

We now turn to proving the following estimate on the drift term in
\prettyref{eq:second-order-is-controlled-k=1}--\prettyref{eq:second-order-is-controlled-k=2}.
Let us begin by first recalling the following useful martingale inequality
due to  \cite{freedman1975tail}, for situations where the almost sure bound on the martingale increments is much larger than the conditional variances: for a submartingale $S_{n}$ with
$\mathbb{E}[(S_{n}-S_{n-1})^{2}\mid\mathcal{F}_{n-1}]\leq K_{2}^{(n)}$
and $|S_{n}-S_{n-1}|\leq K_{1}$ a.s., we
have 
\begin{align*}
\mathbb{P}\Big(S_{t}\leq-\lambda\Big)\leq & \exp\Big(\frac{-\lambda^{2}}{\sum_{n\leq t}K_{2}^{(n)}+\frac{1}{3}K_{1}\lambda}\Big)\,.
\end{align*}
With this in hand, we can estimate the drift as follows.
\begin{prop}
	\label{prop:drift-bound}If $k=1$, for every $t\leq M$,
	and every $K>0$, we have
	\begin{align*}
	\inf_{x_0}\prob_{x_{0}}\Big(\sum_{j=0}^{t-1}\frac{\delta a_{k}}{4N}\Big(1-\frac{L_{j+1}}{K\bar{L}}\Big)\geq & \sum_{j=0}^{t-1}\frac{\delta a_{k}}{8N}\Big)\geq1-\frac{2}{K}\,.
	\end{align*}
	Let $\hat{L}=N^{\frac{1}{2}-\frac{\iota}{4}}$. If $k\geq2$,
	if $\alpha\delta^{2}\leq\varepsilon$ for some $\varepsilon=O(1)$, $\delta\leq1$, and $\alpha$ is of at most polynomial growth
	in $N$, then for every $\gamma>0$, 
	\begin{align*}
	\lim_{N\to \infty}\inf_{x_{0}}\mathbb{P}_{x_{0}}\bigg(\sum_{j=0}^{t-1}\frac{\delta a_{k}|m_{j}|^{k-1}}{4N}\Big(1-\frac{L_{j+1}\mathbf{1}_{\{L_{j+1}<\hat{L}\}}}{K\bar{L}\log N}\Big)\geq-\frac{\gamma}{10\sqrt{N}}+\sum_{j=0}^{t-1}\frac{\delta a_{k}|m_{j}|^{k-1}}{8N}\,,\quad\!\!\forall t\leq M\bigg) & =1.
	\end{align*}
	Moreover, this limit holds uniformly over choices of $D$ and $\varepsilon=O(1)$.
\end{prop}

\begin{proof}
	We begin with the case of $k=1$. Observe that for every $t\leq M$,
	we have for every $x_{0}\in\mathbb{S}^{N-1}$, 
	\begin{align*}
	\mathbb{P}_{x_{0}}\Big(\sum_{j=0}^{t-1}\Big(\frac{L_{j+1}\mathbf{1}_{\{L_{j+1}<\hat{L}\}}}{K\bar{L}}-\frac{1}{2}\Big)>0\Big)
	\leq & \frac{2\sup_{x}\mathbb{E}[|\frac{1}{\sqrt N}\nabla H(x)|^{2}]}{K\bar{L}}\leq\frac{2}{K}\,.
	\end{align*}
	As this bound is independent of $t$ and $x_{0}$, we obtain the desired.
	Thus with probability $1-\frac{2}{K}$, 
	\[
	\sum_{j=0}^{t-1}\frac{L_{j}}{K\bar{L}}\leq\frac{t}{2}\,,\qquad\mbox{and}\qquad\sum_{j=0}^{t-1}\frac{\delta a_{k}}{4N}\Big(1-\frac{L_{j+1}}{K\bar{L}}\Big)\geq\sum_{j=0}^{t-1}\frac{\delta a_{k-1}}{8N}\,.
	\]
	
	In the case $k\geq2$, it suffices to prove the following: for $\hat{L}=N^{\frac{1}{2}-\frac{\iota}{4}}$
	and $\alpha$ of at most polynomial growth with $\alpha\delta^{2}\leq\varepsilon$,
	we have for every $\gamma>0$,
	\begin{align}
	\lim_{N\to\infty}\sup_{x_{0}\in\mathbb{S}^{N-1}}\mathbb{P}_{x_{0}}\Big(\inf_{t\leq M}\sum_{j=0}^{t-1}\frac{\delta a_{k}|m_{j}|^{k-1}}{4N}\Big(\frac{1}{2}-\frac{L_{j+1}\mathbf{1}_{\{L_{j+1}<\hat{L}\}}}{K\bar{L}\log N}\Big)<-\frac{\gamma}{10\sqrt{N}}\Big) & =0\,,\label{eq:k=2-drift-bound}
	\end{align}
	Fix any $x_{0}\in\mathbb S_{N-1}$; everything
	that follows will be uniform in $x_{0}\in\mathbb{S}^{N-1}$. To this end, observe that by a union bound,
	the desired probability in~\eqref{eq:k=2-drift-bound} is upper-bounded by
	\begin{align*}
	\mathbb{P}_{x_{0}}\Big(\inf_{t\leq M}\sum_{j=0}^{t-1} & \frac{\delta a_{k}|m_{j}|^{k-1}}{4N}  \Big(\frac{1}{2}-\frac{L_{j+1}\mathbf{1}_{\{L_{j+1}<\hat{L}\}}}{K\bar{L}\log N}\Big)<-\frac{\gamma}{10\sqrt{N}}\Big)\\
	& \leq M\sup_{t\leq M}\mathbb{P}_{x_{0}}\Big(\sum_{j=0}^{t-1}\frac{\delta a_{k}|m_{j}|^{k-1}}{4N}\Big(\frac{1}{\log N}-\frac{L_{j+1}\mathbf{1}_{\{L_{j+1}\leq\hat{L}\}}}{K\bar{L}\log N}\Big)<-\frac{\gamma}{10\sqrt{N}}\Big).
	\end{align*}
	Recall the filtration $\mathcal{F}_{j}$. For every $j$, $m_{j}$
	is $\cF_{j}$-measurable and 
	\begin{align*}
	\mathbb{E}\Big[L_{j+1}\mathbf{1}_{\{L_{j+1}\leq\hat{L}\}}\mid\cF_{j}\Big] & \leq\sup_{x}\mathbb{E}\Big[\Big|\frac{1}{\sqrt N}\nabla H(x)\Big|^{2}\Big]\leq\bar{L}
	\end{align*}
	so that for $K\geq1$, the sum 
	\[
	Z_{t}:=\sum_{j=0}^{t-1}\frac{\delta a_{k}|m_{j}|^{k-1}}{4N}\Big(\frac{1}{\log N}-\frac{L_{j+1}\mathbf{1}_{\{L_{j+1}\leq\hat{L}\}}}{K\bar{L}\log N}\Big)
	\]
	is an $\cF_{t}$-submartingale. Observe that $|Z_{t}-Z_{t-1}|\leq\frac{\delta a_{k}}{8N}\left(\frac{1}{\log N}\vee\frac{\hat{L}}{\bar{L}\log N}\right)$
	almost surely and that the conditional variances are bounded by 
	\[
	\mathbb{E}\Big[\Big(\frac{\delta a_{k}|m_{j}|}{8N}\Big)^{2}\Big(\frac{1+L_{j+1}}{\bar{L}\log N}\Big)^{2}\Big|\cF_{j}\Big]\leq 2 \Big(\frac{\delta a_{k}}{8N\log N}\Big)^{2}\Big(1+\frac{\sup_{x}\mathbb{E}[|\frac{1}{\sqrt N}\nabla H(x)|^{4}]}{\bar{L}^{2}}\Big)\leq\Big(\frac{\delta a_{k}}{4N\log N}\Big)^{2}
	\]
	almost surely. Thus we may apply Freedman's inequality to obtain 
	\begin{align*}
	\sup_{T\leq M}\mathbb{P}_{x_{0}}\Bigg(\sum_{j=0}^{T-1}\frac{\delta a_{k}|m_{j}|^{k-1}}{4N}& \Big(\frac{1}{2}-\frac{L_{j+1}\mathbf{1}_{\{L_{j+1}\leq\hat{L}\}}}{K\bar{L}\log N}\Big)<-\frac{\gamma}{10\sqrt{N}}\Bigg) \\
	& \leq \exp\Bigg(\frac{-\gamma^{2}/(100\cdot N)}{M\Big(\frac{\delta a_{k}}{4N\log N}\Big)^{2}+\frac{1}{3}\frac{\gamma}{\sqrt{N}}\frac{\delta a_{k}}{8N\log N}(1+\hat{L}/\bar{L})}\Bigg).
	\end{align*}
	Since $\alpha\delta^{2}\leq\varepsilon$ , the first term in the denominator
	is $o(\frac{1}{N(\log\alpha N)})$ as $\alpha$ is not growing faster
	than polynomially in $N$. If we then take any $\hat{L}$ such that
	\[
	\frac{\delta\gamma\hat{L}}{\sqrt{N}\log N}=o\Big(\frac{1}{\log\alpha N}\Big),
	\]
	the entire bound is $o(\frac{1}{M})$, and a union bound over all
	$T\leq M$ implies that the probability is $o(1)$ uniformly in the
	choice of $x_{0}$. The above criterion on $\hat{L}$ is satisfied
	with the choice $\hat{L}=N^{\frac{1}{2}-\frac{\iota}{4}}$ provided
	$\delta\leq1$ and that $\alpha$ is of at most polynomial growth
	in $N$, implying the desired \prettyref{eq:k=2-drift-bound}. 
\end{proof}

\subsection{Controlling the directional error martingale\label{subsec:directional-error-martingale}}

It remains to bound the effect of the directional error martingale from \eqref{eq:DE-martingale}. 
Recall that for every initialization $X_{0}=x_{0}\in\mathbb{S}^{N-1}$,
the point $X_{t}$ is $\cF_{t}$ measurable. This implies that for
each $t$, the increment 
\[
M_{t}-M_{t-1}={-}\frac{\delta}{{N}}\nabla H^{t}(X_{t-1})\cdot e_{1}\,,
\]
has mean zero conditionally on $\cF_{t-1}$---for every $x$, by
definition of the sample error, $\mathbb{E}[H(x)]=0$ and
$\mathbb{E}[\nabla H(x)]=(0,...,0)$---so that $M_{t}$ is indeed
an $\mathcal F_t$-adapted martingale. To control the fluctuations of this martingale, we recall
Doob's maximal inequality: if $S_{t}$ is a submartingale, then for every $p\geq1$, 
\begin{align*}
\mathbb{P}\Big(\max_{t\leq T}S_{t}\geq\lambda\Big)\leq & \frac{p\mathbb{E}[|S_{T}|^{p}]}{(p-1)\lambda^{p}}\,.
\end{align*}
By \prettyref{eq:condition-1}, we then deduce the following. 
\begin{lem}
	If $C_{1}$ is as in Definition~\ref{def:naturally-scaling}\,
	for every $r>0$, we have 
	\begin{align}
	\sup_{T\leq M}\sup_{x_{0}\in\mathbb{S}^{N-1}}\mathbb{P}_{x_{0}}\left(\max_{t\leq T}\frac{1}{\sqrt{T}}\Big|\sum_{j=0}^{t-1}\nabla H^{j+1}(X_{j})\cdot e_{1}\Big|\geq r \right) & \leq\frac{2C_{1}}{r^{2}}\,,\label{eq:sample-error-control}
	\end{align}
\end{lem}

\begin{proof}
	For convenience, let $\tilde{M}_{t}=N M_{t}/\delta$. Observe
	that $\tilde{M}_{t}$ is a martingale with variance
	\begin{align*}
	\sup_{x_{0}}\mathbb{E}_{x_{0}}[\tilde{M}_{t}^{2}]=\sup_{x_{0}}\mathbb{E}_{x_{0}}\Big[\Big(\sum_{j=0}^{t-1}\nabla H^{j+1}(X_{j})\cdot e_{1}\Big)^{2}\Big]\leq & t\sup_{x}\mathbb{E}[(\nabla H(x)\cdot e_{1})^{2}]\leq t C_{1}\,.
	\end{align*}
	Thus, by Doob's maximal inequality, we have the desired bound,
	\[
	\sup_{x_{0}}\mathbb{P}_{x_{0}}\Big(\sup_{t\leq T}|\tilde{M}_{t}|>r\sqrt{T}\Big)\leq  \frac{2\sup_{x_{0}}\mathbb{E}_{x_{0}}[(\tilde{M}_{T})^{2}]}{r^2 T}\leq\frac{2C_{1}}{r^{2}}\,. \qedhere
	\]
\end{proof}
By \prettyref{eq:sample-error-control}, for every  $d>0$, 
\begin{equation}
\sup_{T\leq M}\sup_{x_{0}}\mathbb{P}_{x_{0}}\Big(\sup_{t\leq T}\Big|\frac{\delta}{{N}}\sum_{j=0}^{t-1}\nabla H^{j+1}(X_{j})\cdot e_{1}\Big|\geq\frac{\delta d\sqrt{T}}{10N}\Big)\leq\frac{200C_{1}}{d^{2}}\,.\label{eq:sample-error-martingale-bound}
\end{equation}

\subsection{Proof of \prettyref{prop:main-difference-inequalities}}
We are now in position to combine the above three bounds and conclude
\prettyref{prop:main-difference-inequalities}. 
Consider first the case $k\geq2$. By \prettyref{prop:drift-bound}
and \prettyref{eq:second-order-is-controlled-k=2} we have that
for every $\gamma>0$ and $\eta<\eta_0(\varrho,a_k,a_{k+1})$,
\begin{align*}
\lim_{N\to\infty}\inf_{x_{0}\in E_{\gamma/\sqrt{N}}}\mathbb{P}_{x_{0}}\bigg(\begin{array}{ll}\!\!m_{t}\geq & \frac{7}{10}m_{0}+\sum_{j=0}^{t-1}\frac{\delta a_{k-1}m_{j}^{k-1}}{8N} \\ 
\vspace{.1cm}
& \quad \quad \quad\!\! \!{-}\frac{\delta}{{N}}\sum_{j=0}^{t-1}\nabla H^{j+1}(X_{j})\cdot e_{1}\end{array}\quad \!\forall t\leq\tau_{\gamma/(2\sqrt{N})}^{-}\wedge\tau_{\eta}^{+}\wedge M\bigg)=1\,.
\end{align*}
(We used here that since $x_{0}\in E_{\gamma/\sqrt{N}}$, and $t\leq\tau_{\gamma/(2\sqrt{N})}^{-}$,
we have $\frac{\gamma}{10\sqrt{N}}\leq\frac{m_{0}}{10}$ and $m_{j}=|m_{j}|$
deterministically.) Furthermore, if $D= D_N$, $\delta \leq \bar \delta(k)$, and $\bar t$ are as in Proposition~\ref{prop:main-difference-inequalities}, for every $x_{0}\in E_{\gamma/\sqrt{N}}$,
if $T\leq\bar{t}$, then 
\[
\frac{\delta D_{N}\sqrt{T}}{10N}\leq\frac{\gamma}{10\sqrt{N}}\leq\frac{m_{0}}{10}.
\]
Thus, applying the directional error martingale bound \prettyref{eq:sample-error-martingale-bound}
with $d=D$, we obtain the desired bound (observing that the $o(1)$ terms are uniform in $D,K$ and $\varepsilon=O(1)$). 

Suppose now that $k=1$. By \prettyref{prop:drift-bound} and \prettyref{eq:second-order-is-controlled-k=1}
we have that for every $K>0$, every $\delta\leq \bar \delta (1)$, and every $N$ sufficiently large, 
\[
\inf_{t\leq M}\inf_{x_{0}\in E_{\gamma/\sqrt{N}}}\mathbb{P}\Big(m_{t}\geq\frac{7}{10}m_{0}+\sum_{j=0}^{t-1}\frac{\delta a_{k-1}}{8N}{-}\frac{\delta}{{N}}\sum_{j=0}^{t-1}\nabla H^{j+1}(X_{j})\cdot e_{1}\,,\quad \!\!\!\forall t\leq\tau_{0}^{-}\wedge\tau_{\eta}^{+}\Big)\geq1-\frac{3}{K} -o(1)\,.
\]
Controlling the directional error martingale by the same argument via \prettyref{eq:sample-error-martingale-bound} with $D= D_N$, we obtain \eqref{eq:k=1-inequality} as desired.

\section{Attaining weak recovery}\label{sec:weak-recovery}

We now turn to proving \prettyref{thm:main-theorem} and the recovery part of Theorem~\ref{thm:k=1}; we invite the
reader to recall the definition of $\bar{t}$ from \eqref{eq:tbar}.
The goal of this section will be to prove that the dynamics will have weakly recovered in some (possibly random) time before $\bar t\wedge M$.

\begin{prop}\label{prop:attain-weak-recovery}
	Under the assumptions of Theorem~\ref{thm:main-theorem}, there exists $\eta_0(\varrho,a_k,a_{k+1})>0$ such that for every $\eta<\eta_0$, for every $\gamma>0$, we have   
	\begin{align*}
	\lim_{N\to\infty} \inf_{x_0 \in E_{\gamma/\sqrt N}} \mathbb P_{x_0} \Big(\tau_{\eta}^{+} \leq \bar t \wedge M\Big)=1\,.
	\end{align*}
	If $k=1$ and we only assume $\alpha\delta^2 \leq \varepsilon = O(1)$ and $\delta$ is sufficiently small, but order one, then 
	\begin{align*}
	\liminf_{N\to\infty} \inf_{x_0\in E_{\gamma/\sqrt N}} \mathbb P_{x_0} \Big(\tau_\eta^+ \leq \bar t \wedge M \Big) \geq 1-C\Big(\delta + \frac{1}{D^2}\Big)\,,
	\end{align*}
	for some constant $C(C_1, a_1, a_2)>0$.  
\end{prop}

The proposition  immediately implies Theorem~\ref{thm:main-theorem}.

\subsection{Proof of Proposition~\ref{prop:attain-weak-recovery}} We analyze the difference inequalities of Proposition~\ref{prop:main-difference-inequalities} when $k=1$, $k=2$, and $k\geq3$
separately. Observe that in all cases, the right-hand side of the
difference inequality is increasing, so that since $x_{0}\in E_{\gamma/\sqrt{N}}$,
we have $\tau_{\gamma/(2\sqrt{N})}^{-}\geq\tau_{\eta}^{+}\wedge \bar t$
necessarily, and we may drop the requirement $t\leq\tau_{\gamma/(2\sqrt{N})}^{-}$
in all cases of \eqref{eq:k=1-inequality}--\prettyref{eq:k>=2-inequality}.

\subsection*{Linear regime: $k=1$. }

The right-hand side of the difference inequality in \eqref{eq:k=1-inequality}
is non-decreasing and greater than $\eta$ at time 
\[
t_{*}:=\lceil\frac{8\eta N}{\delta a_{k}}\rceil.
\]
As such, $(\tau_{\eta}^+\wedge \bar t \wedge M) \leq t_\star$ necessarily, and as long as $t_{*}<\bar{t}\wedge M$, we will have the desired
\begin{align*}
\lim_{N\to\infty}\inf_{x_{0}\in E_{\gamma/\sqrt{N}}}\mathbb{P}_{x_0}\Big(\tau_{\eta}^+\leq \bar t \wedge M \Big)= & 1\,.
\end{align*}
In order for $t_{*}<\bar{t}\wedge M$, we need, for a sequence $D_{N}\uparrow\infty$
arbitrarily slowly, 
\[
\lceil\frac{8\eta}{\delta a_{k}}\rceil\leq\frac{\gamma^{2}}{D_{N}^{2}\delta^{2}}\wedge\alpha.
\]
We see that this inequality is satisfied for
the choices of $\alpha,\delta$ as in the statement of Theorem~\ref{thm:main-theorem},
since $\delta=o(1)$ (as we can choose $D_{N}\to\infty$
arbitrarily slowly, so that $D_{N}^{2}\delta=o(1)$) and $\delta\alpha\uparrow\infty$. 

In the case of $\alpha = O(1)$, we can take $D$ sufficiently large, and subsequently take $\delta$ sufficiently small (depending on the $D$) such that the above inequality is satisfied.  
\qed

\subsection*{Quasilinear regime: $k=2$. }

In this case, we may use the discrete Gr\"ownwall inequality:  suppose that $\left(m_{t}\right)$ is any sequence such that for some $a,b\geq 0$ 
\begin{equation}\label{eq:disctete-gronwall}
m_{t}\geq a+\sum_{\ell=0}^{t-1}b m_{\ell}  \qquad \text{ then }\qquad
m_{t}\geq a(1+b)^{t}\,.
\end{equation}

Define the lower-bounding function, 
\begin{align*}
g_{2}(t):= & \frac{m_{0}}{2}\exp\Big(\frac{\delta a_{k}}{8N}t\Big)\,.
\end{align*}
Applying the discrete Gr\"onwall inequality to  \prettyref{eq:k>=2-inequality}, we obtain that for every $\gamma>0$, 
\[
\lim_{N\to\infty}\inf_{x_{0}\in E_{\gamma/\sqrt{N}}}\mathbb{P}_{x_{0}}\Big(m_{t}\geq g_{2}(t)\quad\mbox{for all }t\leq\tau_{\eta}^{+}\wedge\bar{t}\wedge M\Big)=1
\]
Since $g_{2}(t)\geq\eta$ for all 
\[
t\geq t_{*}:=\Big\lceil\frac{8N}{\delta a_{k}}(\log\frac{2}{m_{0}}+\log\eta)\Big\rceil\,,
\]
we see that as long as $t_{*}\leq\bar{t}\wedge M$, we will have $\lim_{N\to\infty}\inf_{x_{0}\in E_{\gamma/\sqrt{N}}}\mathbb{P}_{x_0} (\tau_{\eta}^+\leq \bar t \wedge M)=1$.
The criterion $t_{*}\leq\bar{t}\wedge M$ is satisfied using the facts
that $\delta=o(\frac{1}{\log N})$ (as we can choose $D_{N}$ to go
to $\infty$ arbitrarily slowly) and $\frac{\alpha\delta}{\log N}\uparrow\infty$.
\qed

\subsection*{Polynomial regime: $k\protect\geq 3$.}
Observe the following discrete analogue of the Bihari-LaSalle inequality: suppose that $(m_{t})$ is a sequence satisfying, for some $k>2$ and $a,b>0$
\begin{equation}\label{eq:Bih-Lasalle}
m_{t}\geq a+\sum_{\ell=0}^{t-1} b m_{\ell}^{k-1}
\qquad \text{ then }\qquad m_{t}\geq\frac{a}{(1-ca^{k-2}t)^{\frac{1}{k-2}}}.
\end{equation}
For the reader's convenience, we provide a proof in Appendix~\ref{app:bihari-lasalle}.

Applying \prettyref{eq:Bih-Lasalle}, we obtain
for every $t\leq\tau_{\eta}^+ \wedge\bar{t}\wedge M$,
\[
m_{t}\geq\frac{m_{0}}{\left(1-\frac{\delta a_{k}}{8N}(k-2)m_{0}^{k-2}t\right)^{\frac{1}{k-2}}}=:g_{k}(t)\,.
\]
In particular, $g_{k}(t)\ge\eta$
provided 
\[
\eta\left(1-\frac{\delta a_{k}}{8N}(k-2)\frac{\gamma^{k-2}}{N^{\frac{k-2}{2}}}t\right)=o\Big(\frac{\gamma^{k-2}}{N^{\frac{k-2}{2}}}\Big)\,.
\]
As such, if  for some $K$ sufficiently large, 
\[
t\geq t_{*}=\Big \lceil \frac{8N}{\delta a_{k}(k-2)\gamma^{\left(k-2\right)/2}}N^{\frac{k-2}{2}}\Big(1-\frac{K}{N^{({k-2})/{2}}}\Big)\Big\rceil\,,
\]
then as long as $t_\ast \leq\bar t \wedge M$, we will have $\lim_{N\to\infty} \inf_{x_0 \in E_{\gamma/\sqrt N}} \mathbb P_{x_0}\big(\tau_\eta^+ \leq \bar t \wedge M\big) = 1$. The criterion $t_\ast\leq \bar t \wedge M$ is satisfied using the facts that $\delta=o(N^{-\frac{k-2}{2}})$
(as we can choose $D_{N}$ to go to $\infty$ arbitrarily slowly)
and $\frac{\alpha\delta}{N^{(k-2)/2}}\uparrow\infty$. \qed

\section{Strong recovery and the descent phase}
We begin this section by proving a law of large numbers for the trajectory in the descent phase,  Theorem~\ref{thm:lln}. We then combine it with Theorem~\ref{thm:main-theorem} to conclude the proof of Theorem~\ref{thm:main-1}.

\begin{proof}[\textbf{\emph{Proof of Theorem~\ref{thm:lln}}}]
	Recall $m_t = m(X_t)$ and $r_t$ and analogously define $\overline m_t : = m(\overline X_{t})$, 
	\begin{align*}
	\overline{m}_{t}=\frac{1}{\overline{r}_{t}}\Big(\overline{m}_{t-1}-\frac{\delta}{{N}}\nabla\Phi(\overline{X}_{t-1})\cdot e_{1}\Big) \qquad \mbox{where} \qquad \overline r_t : = \sqrt{1 + \delta^2 |\nabla \Phi(\overline X_{t-1})|^2/N^2}
	\end{align*}
	As shown in~\eqref{eq:r-bound}, $|r_t -1|\le \delta^2 (\frac{A}{N^2} + \frac{L_t}{N})$ where $L_{t}=|\nabla H^{t}(X_{t-1})/\sqrt{N}|^{2}$. By similar reasoning, 
	$|\bar r_t - 1| \le \frac{A \delta^2}{N^2}$. At the same time, 
	\begin{align*}
	\big|m_{t-1} -  & \frac{\delta}{ N} \nabla \Phi(X_{t-1})\cdot e_1 - \frac{\delta}{ N} \nabla H^t(X_{t-1})\cdot e_1 \big| \\
	& \qquad \le 1 + \sup_{x}\frac{\delta}{ N} |\nabla \Phi(x) \cdot e_1|+ \frac{\delta}{ N}|\nabla H^t(X_{t-1}) \cdot e_1|  \le 1+ \frac{\delta\sqrt A}{N} + \frac{\delta\sqrt{L_t}}{\sqrt N}\,.
	\end{align*}
	As such, we have the bound
	\begin{align*}
	\Bigg|m_{t}-\bigg(m_{t-1}-\frac{\delta}{{N}}\nabla\Phi(m_{t-1})\cdot e_{1}-\delta\frac{1}{{N}}\nabla H^{t}(X_{t-1})\cdot e_{1}\bigg)\Bigg|\leq & \frac{2\delta^{2}(1\vee L_{t})(1\vee \frac{\delta\sqrt{L_t}}{\sqrt N})}{N}\,.
	\end{align*}
	Iterating the above bound, we see that 
	\begin{align*}
	\sup_{t \le M} \Bigg|m_{t}-\bigg(m_{0}-\sum_{\ell=0}^{t-1}\frac{\delta}{{N}}\nabla\Phi(m_{\ell})\cdot e_{1}-\frac{\delta}{{N}}\sum_{\ell=0}^{t-1}\nabla H^{\ell+1}(X_{\ell})\cdot e_{1}\bigg)\Bigg|\leq & \sum_{\ell=0}^{M-1}\frac{2\delta^{2}(1\vee L_{\ell})(1 \vee  \frac{\delta \sqrt L_\ell}{\sqrt N})}{N}\,.
	\end{align*}
	Consider the probability that the quantity on the right-hand side
	is large: for every $\varepsilon>0$, 
	\begin{align*}
	\mathbb{P}\Big(\sum_{\ell=0}^{M-1}\frac{\delta^{2}(L_{\ell}+ \frac{\delta\sqrt L_\ell}{\sqrt N}+ \frac{\delta L_\ell^{3/2}}{\sqrt N})}{N}>\varepsilon\Big)&  \leq {\varepsilon^{-1} \alpha \delta^2 \sup_{\ell}\mathbb E\Big[L_\ell+ \frac{\delta\sqrt L_\ell}{\sqrt N} + \frac{\delta L_\ell^{3/2}}{\sqrt N}\Big]}\,.
	\end{align*}
	Recalling from~\eqref{eq:bar-l-def} that $\sup_{\ell}\mathbb E[L_\ell]\vee \mathbb E[L_\ell^2]\le \bar L$, we see that each of the expectations above are bounded by $\bar L$,
	whereby as long as $\alpha \delta^2= o(1)$, this is $o(1)$ for all fixed $\varepsilon$.
	
	As such, it suffices to consider the linearization
	\begin{align*}
	\mathsf{m}_{t}:= & m_{0}-\sum_{\ell=0}^{t-1}\frac{\delta}{{N}}\nabla\Phi(m_{\ell})\cdot e_{1}-\frac{\delta}{{N}}\sum_{\ell}^{t-1}\nabla H^{\ell+1}(X_{\ell})\cdot e_{1}
	\end{align*}
	as for every $\varepsilon$, with probability $1-o(1)$, we have $\sup_{\ell\leq M} |\mathsf{m}_{\ell}-m_{\ell}|\leq \varepsilon$.
	
	By the same reasoning, as long as $\alpha \delta^2 = o(1)$, for every $\varepsilon$, for $N$ large enough, we have deterministically $|\bar {\mathsf m}_\ell - \bar m_\ell|\le \varepsilon$, where $\bar {\mathsf m}$ is the linearized population dynamics,
	\begin{align*}
	\bar{\mathsf m}_t = m_0 - \sum_{\ell = 0}^{t-1} \frac{\delta}{ N} \nabla \Phi(m_\ell)\,.
	\end{align*}
	
	With the above in hand, it clearly suffices to show that  
	\begin{align}\label{eq:wts-lln}
	\sup_{\ell \le M} |\bar {\mathsf m}_\ell - \mathsf m_\ell| \to 0\qquad \mbox{in $\mathbb P$-prob}\,.
	\end{align} 
	Towards that, let us control the effect of the directional error martingale for all times. Recall from Doob's maximal inequality as in~\eqref{eq:sample-error-martingale-bound} that for every $\lambda>0$,
	\begin{align}\label{eq:lln-martingale}
	\mathbb{P}\Big(\sup_{t\leq M}\frac{\delta}{N}\Big|\sum_{\ell = 0}^{t-1}\nabla H^{\ell+1}(X_{\ell})\cdot e_{1}\Big|>\lambda\Big)\leq & \frac{C_{1} \alpha \delta^2 }{\lambda^2 N}=o(N^{-1})\,.
	\end{align} 
	To show~\eqref{eq:wts-lln}, consider the probability that the supremum is greater than some $\gamma>0$, and split the supremum into a one over $\ell \le T\delta^{-1} N$ and one over $T\delta^{-1}N \le \ell \le M$ for a $T$ to be chosen sufficiently large. 
	For the former, fix any $T$, and recall that $\nabla\Phi(m)\cdot e_{1}=\phi'(m)(1-m^{2})$. As $\phi'_N$ are uniformly $C^1$ on $[0,1]$, there exists $K$ such that uniformly over $N$,
	\begin{align*}
	\sup_{x,y\in [0,1]}|\phi'(x)(1-x^2) - \phi'(y)(1-y^2)| \le K|y-x|\,.
	\end{align*}
	We obtain from
	this that for every $\varepsilon>0$, on the event  $\{\sup_{\ell \le M} |m_\ell- \mathsf m_\ell|\vee |\bar m_\ell - \bar{\mathsf m}_\ell| <\varepsilon\}$, 
	\begin{align*}
	|\bar{\mathsf m}_{t}-\mathsf m_{t}| &\leq \sum^{t-1}_{\ell=0}\frac{\delta}{ N}|\nabla \Phi(\bar m_\ell)\cdot e_1- \nabla \Phi(m_\ell)\cdot e_1|+\frac{\delta}{{N}}\Big|\sum_{\ell = 0}^{t-1}\nabla H^{\ell+1}(X_{\ell})\cdot e_{1}\Big| \\ 
	& \le \sum_{\ell=0}^{t-1}\Big[\frac{\delta K}{N} |\bar{\mathsf m}_\ell - \mathsf m_\ell| + \frac{2\delta K \varepsilon}{N} \Big] + \frac{\delta}{{N}}\Big|\sum_{\ell=0}^{t-1}\nabla H^{\ell+1}(X_{\ell})\cdot e_{1}\Big|\,.
	\end{align*}
	Combined with~\eqref{eq:lln-martingale}, as long as $\alpha \delta^2 = o(1)$, with probability $1-o(1)$, we have for all $t\le T\delta^{-1} N$,  
	\begin{align*}
	|\bar{\mathsf m}_{t}-\mathsf{m}_{t}|\leq & (2TK+1)\varepsilon+\frac{\delta K}{N}\sum_{\ell = 0}^{t-1}|\bar{\mathsf m}_{\ell}-\mathsf{m}_{\ell}|\,,
	\end{align*}
	which by the discrete Gronwall inequality, implies that for every
	$\varepsilon>0$, with probability $1-o(1)$, 
	\begin{align*}
	\sup_{t\leq T\delta^{-1} N}|\mathsf m_{t}-\bar{\mathsf m}_{t}|\leq & (\frac{2M K}{N}\delta+1)\varepsilon e^{KT}\,.
	\end{align*}
	For each $\gamma>0$, there exists $\varepsilon(K,T)>0$ such that the above is at most $\gamma/5$.

	Now let $T= T(\gamma)$ be such that $$\sup_{T\delta^{-1} N \le t\le N}|\bar{\mathsf m}_{t} -1|<\frac{\gamma}{5}\,;$$ this $T$ exists and is $O(1)$ by Assumption A. In that case, using again Assumption A to note that $\nabla \Phi(m)\ge 0$ while $m\ge 0$, we get 
	\begin{align*}
	\sup_{T \delta^{-1} N \le t\le M}|\mathsf m_t - \bar{\mathsf m}_t| < \frac {2\gamma}{5} + |\mathsf m_{T\delta^{-1}N} - \bar {\mathsf m}_{T\delta^{-1} N}| + \frac{\delta}{N} \Big|\sum_{\ell = T\delta^{-1} N}^{M} \nabla H^{\ell+1}(X_\ell) \cdot e_1\Big|
	\end{align*}
	By the first part of~\eqref{eq:wts-lln}, and the bound of~\eqref{eq:lln-martingale} applied to $\gamma/5$, we see that the probability that the above is greater than $\gamma$ is also $o(1)$. Together these yield~\eqref{eq:wts-lln}. 
\end{proof}

\subsection{Proof of \prettyref{thm:main-1} and item (a) of~\prettyref{thm:k=1}}
Let $(\alpha_N,\delta_N)$ be as in Theorem~\ref{thm:main-1}, fix any $\varepsilon>0$, and consider the $\mu_N\times\mathbb P$-probability that $|m(X_M)- 1|>\varepsilon$; we need to show this goes to zero as $N\to\infty$. By the Poincar\'e Lemma, for every $\zeta>0$, there exists $\gamma$ such that for all $N$ sufficiently large, 
\begin{align*}
\mu_N(m_0 < \gamma/\sqrt N) < \zeta/3\,.
\end{align*}
Let us now suppose that $X_0$ is such that $m_0\ge \gamma/\sqrt N$. By Theorem~\ref{thm:main-theorem}, there exists $\eta_0>0$ such that for any such $X_0$, for $N$ sufficiently large, we have 
\begin{align*}
\mathbb P(\tau_\eta^+ <  M_1) < \zeta/3
\end{align*}
for $M_1 = M/2$ (notice that the criteria of the theorem apply equally whether we take $\alpha$ or $\alpha/2$). 

Observe that by the Markov property, conditionally on the stopping time $\tau_\eta^+$ and the value $X_{\tau_\eta^+}$. we have the distributional equality 
\begin{align*}
\mathbb P_{x_0}(X_{\tau_\eta^++s} \in \cdot \mid \tau_\eta^+,X_{\tau_\eta^+}) \stackrel{d}= \mathbb P_{X_{s}} (X_{s}\in \cdot)
\end{align*}
As such, for $X_0$ satisfying $m(X_0)\ge \gamma/\sqrt N$, we have
\begin{align*}
\mathbb P_{X_0} (m(X_M)\le 1-\varepsilon) \le \sup_{M_1\le s\le M}\sup_{y_0: m(y_0)\ge \eta} \mathbb P_{y_0}(m(X_s) \le 1-\varepsilon) + \frac{2\zeta}{3}\,.
\end{align*}
Using again the fact that the criteria in Theorem~\ref{thm:lln} on $\alpha, \delta$ apply equally well if we replace $\alpha$ with $\alpha/2$, we see that for $N$ sufficiently large, the right-hand side above is bounded by
\begin{align*}
\sup_{M_1 \le M_2\le M} \zeta + \mathbf 1\{m(\bar X_{M_2})<1-\frac{\varepsilon}{2}\}\,.
\end{align*}
Using the Assumption A, the indicator function above is $0$ as long as $M_2$ is a sufficiently large constant depending on $\varepsilon$, which it necessarily is since it is at least $M/2$ and $M = \omega (1)$ by the assumptions of Theorem~\ref{thm:main-1}.

The $k=1$ case with linear sample complexity of item (a) of Theorem~\ref{thm:k=1} follows naturally by noticing that $\alpha$ could have been taken to be a large enough constant in the above at the expense of $\varepsilon$ and $\zeta$ being small but order one. \qed

\section{Online SGD does not recover with smaller sample complexity}

Here, we prove an accompanying refutation theorem, showing that if $\alpha$ is smaller than in Theorem~\ref{thm:main-theorem} (up to factors of $\log N$), there is not enough time for the online SGD to weakly recover in one pass through the $M$ samples. 

\begin{proof}[\emph{\textbf{Proof of \prettyref{thm:refutation-1} and item (b) of \prettyref{thm:k=1}}}]
	We will in fact prove the following stronger refutation: for every $\eta>0$, and every $\Gamma>0$, 
	\begin{align}\label{eq:wts-refutation}
	\sup_{x:m(x)<\Gamma N^{-1/2}} \mathbb P_{x}\Big(\sup_{t\le M} m(X_t)>\eta\Big) = o(1)\,.
	\end{align}
	This implies Theorem~\ref{thm:refutation-1} because for every $\varepsilon$ there exists a $\Gamma$ such that $\mu_N(m(x)>\Gamma N^{-1/2})<\varepsilon$ for all $N$ sufficiently large, by the Poincar\'e lemma.
	
	Recall that the radius $r_t = |\tilde X_t|$ satisfies $r_t \geq 1$ deterministically, so that by~\eqref{eq:m-t-finite-difference}, as long as $m_t >0$, we have
	\begin{align*}
	m_t \leq m_{t-1} {-} \frac{\delta}{ N} \nabla \Phi (X_{t-1})\cdot e_1 {-} \frac{\delta}{ N} \nabla H^t (X_{t-1}) \cdot e_1\,.
	\end{align*}
	From this, we can first observe the following crude bound on the maximal one-step change $m_t - m_{t-1}$. If $m_{t-1}<0$ but $m_t>0$, then we have the above inequality without the $m_{t-1}$, and furthermore we can put absolute values on each of those terms. Recall that for every $r>0$, 
	\begin{align*}
	\sup_{x\in \mathbb S^{N-1}} \mathbb P\Big( \Big|\frac{\delta}{ N} \nabla H(x)\cdot e_1\Big| >r \Big) \leq \frac{\delta^2 C_1}{N^2 r^2}
	\end{align*}
	from which, for every sequence $d_N>0$ going to infinity arbitrarily slowly, and every  $t>0$, 
	\begin{align*}
	\sup_{x: m(x)\leq 0} \mathbb P \Big( m_{t} > d_N N^{-1/2} \mid X_{t-1} = x\Big) \leq \frac{C_1 \alpha \delta^2} { M d_N^2} = O(\frac{\alpha \delta^2}{M d_N^2})
	\end{align*} 
	By the Markov property of the online SGD, by a union bound over the $M$ samples and using the fact that $\alpha \delta^2 = o(1)$, we may, with probability $1- o(\alpha \delta^2/(d_N^2))$, work under the event that $|\delta (\nabla H^{j+1}(X_j)\cdot e_1)/ N|< d_N/\sqrt{N}$  for every $j$. 
	
	With that in mind, summing up the finite difference inequality over all time, we see that for every $t\leq \tau_0^-\wedge \tau_\eta^+$, we have 
	\begin{align*}
	m_t \leq  m_0 + \frac{\delta}{ N} \sum_{j = 0}^{t-1} \frac{3a_k}{2 N} m_j^{k-1} {-} \frac{\delta}{ N} \sum_{j=0}^{t-1} \nabla H^{j+1}(X_j)\cdot e_1\,.
	\end{align*} 
	Recall from Doob's maximal inequality, that 
	\begin{align}\label{eq:doob-inequality-refutation}
	\sup_{x_0} \mathbb P_{x_0}\Big(\sup_{1 \leq s\leq M} \Big| \frac{\delta}{ N}\sum_{j=0}^{s-1} \nabla H^{j+1} (X_j ) \cdot e\Big| > r\Big) \leq \frac{2 \alpha \delta^2  C_1}{Nr^2}\,.
	\end{align}
	Now define a sequence of excursion times as follows: for $i\geq 1$, we let $\tau_0 = 0$ and 
	\begin{align*}
	\tau_{2i-1}  := \inf\{t\geq \tau_{2i-2}: m(X_t)\leq 0\} 
	\qquad\tau_{2i}   : = \inf\{t\geq \tau_{2i-1}: m(X_t)>0\}
	\end{align*}
	Taking $r = d_N N^{-1/2}$ in~\eqref{eq:doob-inequality-refutation}, it follows that
	\begin{align*}
	\inf_{x : m(x)<\Gamma N^{-1/2}} \mathbb P\Big( m_t \leq m_{\tau_{2i}}+ \frac{d_N}{\sqrt N} +  \frac{\delta}{\sqrt N} \sum_{j=0}^{t-1}\frac{3a_k}{2\sqrt N}m_j^{k-1}, \quad  \!\!\forall i,\,\forall t\leq [\tau_{2i},& \tau_{2i+1}\wedge \tau^+_\eta]\Big) \\ 
	& = 1-O(\frac{\alpha \delta^2}{d_N^2})\,.
	\end{align*}
	We next claim that the inequality above implies, deterministically, that through every excursion (i.e., for all $i$), we have $\tau_\eta^+> \tau_{2i+1}\wedge M$. First of all, recall that we have restricted to the part of the space on which $m_{\tau_{2i}}$ is always less than $d_N N^{-1/2}$.  Then, if we have the inequality in the probability above, by the discrete Gronwall inequality \eqref{eq:disctete-gronwall} when $k=2$ and the discrete Bihari--LaSalle inequality \eqref{eq:Bih-Lasalle} when $k>2$, we have for some $c = c(k)>0$, 
	\begin{align*}
	m_t \leq \begin{cases} 
	2d_N N^{-1/2}+ \frac{2\delta a_k}{N} t & k=1 \\
	2d_N N^{-1/2}  \exp \Big(\frac{2\delta a_k}{N} t\Big) & k=2 \\ 
	2d_N N^{-1/2} \big( 1- c  \delta a_k d_N^{k-2} N^{-\frac{k-2}2 -1} t\big)^{- 1/(k-2)} & k\geq 3\end{cases}
	\end{align*} 
	For $N$ sufficiently large, the right-hand side above is smaller than $\eta$ for all $t\leq \tilde t$ where 
	$$
	\tilde t = \begin{cases}
	\epsilon \eta \delta^{-1} N  & k=1 \\
	\frac{\epsilon}{\delta} N \log N  & k=2 \\ 
	d_N^{-\epsilon}   \delta^{-1} N^{1+\frac{k-2}2} & k\geq 3
	\end{cases}
	$$
	for some $\epsilon>0$ sufficiently small depending on $k, a_k, a_{k+1}, C_1$ (as long as $d_N$ is growing slower than $N^{\frac 12-\zeta}$ for some $\zeta>0$ say). Recall the restrictions on $\alpha$ and let $b_N$ be a sequence going to infinity arbitrarily slowly.  
	\begin{enumerate}
		\item If $k=1$, we can choose $d_N$ to be any diverging sequence. Since $\alpha = o(1)$ and $\delta=O(1)$, the above probabilities were all $1-o(1)$, and for every fixed $\eta>0$, we have $\tilde t \geq M$. 
		\item If $k=2$, we can choose $d_N$ diverging as a power in $N$, say $N^{1/4}$ such that for all $\alpha \delta^2 = o(N^{1/2})$, and in particular $\delta = O(1)$, the above probabilities $1-O(\alpha \delta^2/d_N^2)$ were all $1-o(1)$ \emph{and} $\tilde t \geq M$.  
		\item If $k>2$, we can choose $d_N$ to be a sequence diverging sufficiently slowly. Then for all $\alpha \delta^2 = O(1)$, the above probabilities were all $1-o(1)$ \emph{and} $\tilde t \geq M$ (where to see this, we combined the inequalities $\sqrt \alpha =o(N^{(k-2)/2})$ and $\sqrt \alpha = O(\delta^{-1})$).  
	\end{enumerate}
	In both cases, we conclude that necessarily $\tau_\eta^+ >\tau_{2i+1}$, and therefore for every $\eta>0$,
	\begin{align*}
	\inf_{x: m(x)<\Gamma N^{-1/2}} \mathbb P \Big( m_t <\eta \quad \mbox{for all } t\in \bigcup_{i} [\tau_{2i-2},\tau_{2i-1}]\Big) = 1-o(1)
	\end{align*}
	At the same time, deterministically, for all $t\in \bigcup_{i} [\tau_{2i-1},\tau_{2i}]$ we have $m(X_t)\leq 0 <\eta$, implying~\eqref{eq:wts-refutation}.
	
	In order to conclude part (b) of Theorem~\ref{thm:k=1} when $k=1$ we reason as above, taking $d_N$ and $\alpha$ to be sufficiently large constants together, then subsequently taking $\delta$ to be a sufficiently small constant so that the probabilities of order $\alpha\delta^2/d_N^2$ above can be made arbitrarily small. 
\end{proof}


\subsection*{Acknowledgements}
The authors thank the anonymous referees for their detailed comments and suggestions. The authors thank Y.\ M.\ Lu for interesting discussions. G.B.A. thanks A. Montanari, Y. Le Cun, and L. Bottou for interesting discussions at early stages of this project. A.J. thanks S. Sen for helpful comments on this work. R.G. thanks the Miller Institute for Basic Research in Science for their support. A.J. acknowledges the support of the Natural Sciences and Engineering Research Council of Canada (NSERC). Cette recherche a \'et\'e financ\'ee par le Conseil de recherches en sciences naturelles et en g\'enie du Canada (CRSNG),  [RGPIN-2020-04597, DGECR-2020-00199].



\appendix

{
	\section{Useful bounds on norms of random vectors}
}
Before getting to the deferred proofs of Section~\ref{sec:examples}, we give a few useful inequalities we will need. 
Recall that if $(Y_\ell)_{\ell=1}^m$ is a collection of $m$ non-negative random variables with finite
$p$-th moment, i.e., $\max_\ell \E Y_\ell^p=K$, and $(p_\ell)_{\ell=1}^{m}$ is such that $\sum p_\ell = p$, then 
\begin{equation}\label{eq:mixed-moment-bound}
\E\big [ Y_1^{p_1}\cdots Y_{m}^{p_m}\big] \leq K.
\end{equation}
This follows by noting that by Young's inequality and Jensen's inequality we have that
\[
\E\Big[(Y_1^{\frac{p_1}{p}}\cdots Y_{m}^{\frac{p_m}{p}})^p\Big] \leq \E\Big[\Big(\sum_{\ell = 1}^{m} \frac{p_\ell}{p} Y_\ell\Big)^p\Big] 
\leq\sum \frac{p_\ell}{p} \E Y_\ell^{p} \leq C\,,
\]
for some $C(K,p)$.

Using the above we can easily obtain the following bound on the moments of the norm of a random vector. Suppose that $X$ is a centered random vector in $\mathbb R^n$ whose entries have uniformly bounded $2k$-th moment, i.e., $\sup_i \E[X_i^{2k}] < K$ for some $k\ge 1$. Then using~\eqref{eq:mixed-moment-bound} for $1\le q \le k$, there is $C = C(K,q)>0$ such that 
\begin{align}\label{eq:vector-norm-bound}
\E[\|X\|_2^{2q}] = \E\Big[\Big(\sum X_i^2\Big)^{q}\Big] \le \E\Big [\sum_{i_1,...,i_q} X_{i_1}^2\cdots X_{i_q}^2\Big ]\le C N^q\,.
\end{align}

\begin{lem}\label{eq:moment-lemma}
	Suppose that $X$ is a centered random vector in $\R^{n}$
	whose entries have uniformly bounded $2k$-th moment, 
	i.e., $\sup_{i}\E [X_{i}^{2k}]<K$
	for some $k\geq1$. Then for $1\leq q\leq 2k$ there is $C=C(K,q)>0$
	such that
	\begin{equation}\label{eq:moment-bound}
	\E[\abs{X\cdot w}^{q}] \leq CK^{\frac{q}{2k}}\norm{w}_{2}^{q}\quad\forall w\in\R^{n}.
	\end{equation}
\end{lem}
\begin{proof}
	First note the following symmetrization inequality (see e.g., Exercise 6.4.5 of~\cite{Vershynin}): if $(\epsilon_i)$ are i.i.d.\ Rademacher random variables, then since $X$ is centered, we have
	\[
	\E\big[\abs{X\cdot w}^{q}\big]\leq2^{q}\E\Big[\abs{\sum_i \eps_i X_iw_i}^{q}\Big]\,.
	\]
	Thus it suffices to assume that the entries of $X$ are jointly symmetric in the 
	sense that that the law of $X$ is invariant under the sign change of any given entry.
	Furthermore, by Jensen's inequality it suffices to consider the case
	that $q=2k$. 
	
	We begin with a direct computation: 
	\[
	\E[\abs{X\cdot w}^{2k}]=\sum_{i_1,...,i_{2k}} w_{i_{1}}\cdots w_{i_{2k}}\cdot\E[ X_{i_{1}}\cdots X_{i_{2k}}].
	\]
	As the entries of $X$ are symmetric, any summand corresponding to an index that appears an odd number of times must be zero.  
	In particular, we have the following.
	Let $\mathcal{P}=\{p_{1},\cdots,p_{2k}\}$
	be a partition of the integer $2k$. We say $\cP$ is \emph{even} if
	the $p_{\ell}$ are even. By \eqref{eq:mixed-moment-bound}, 
	all of these moments are of the form $\E X_{j_1}^{p_1}\cdots X_{j_m}^{p_m} \leq K$
	for some even partition $\cP$.
	As such, the above sum is bounded by
	\[
	\sum_{\mathcal{P}\:even}c(\mathcal{P})K\prod_{p\in\mathcal{P}}\norm{w}_{p}^{p}\le  \sum_{\mathcal{P}\:even}c(\mathcal{P})K\norm{w}_{2}^{2k},
	\]
	where here $c(\mathcal{P})$ is the number of groupings of $2k$ items in to $m$ groups of sizes $p_1,\ldots,p_m$.
	Here we used that since $\mathcal P$ is even, $p\geq 2$ for any $p\in\cP$ so that $\norm{w}_{p}\leq\norm{w}_{2}$,
	and that $\sum p_{\ell}=2k$ as $\mathcal{P}$ is a partition. In
	particular, as $\max c(\mathcal P)$, and the number of even partitions of $2k$, each depend only on $k$ we get 
	\[
	\E[\abs{X\cdot w}^{2k}]\leq C\cdot K\cdot\norm{w}_{2}^{2k}\,,
	\]
	for some $C=C(K,k)>0$ as desired.
\end{proof}

\section{Deferred proofs from Section~\ref{sec:examples}}\label{app:deferred-proofs}
In this section, we verify that the various examples of Section~\ref{sec:examples} satisfy Assumptions A--B. 

\subsection{Proof of \prettyref{prop:GLM-pop}}\label{app:Generalized-linear-models}

We begin with the to the proof of \eqref{eq:GLM-pop}. Fix $f$ as in the statement of the theorem. Since $f'$ is of at most polynomial growth, so is $f$. In particular, $f\in L^2(\varphi)$ where $\varphi$ is the standard Gaussian measure on $\R$. Recall \eqref{eq:GLM-pop-loss-1}, and let $C_\epsilon = \E[\epsilon^2]$. By rotational invariance of the Gaussian ensemble, we may take $v_{0}=e_{1}$ there. Furthermore,  the $x$-dependence of the population loss depends only on $x$ through $x\cdot e_1$, so that 
\begin{equation}\label{eq:Phi-GLM-rotation-invariance}
\Phi(x)=\E\left[\left(f\left(a_{1}m(x)+a_{2}\sqrt{1-(m(x))^{2}}\right)-f(a_{1})\right)^{2}\right] + C_\epsilon
\end{equation}
where $a_1,a_2\sim\cN(0,1)$ are independent.

To compute this expectation, recall the following.
For $s\in[-1,1]$, consider the Noise operator $T_s:L^2(\varphi)\to L^2(\varphi)$
\[
T_s f(x) = \E[ f(x s+\sqrt{1-s^2} a_2)]\,.
\]
Recall that the Hermite polynomials satisfy $T_s h_k = s^k h_k$ \cite{LedouxTalagrand}. (Usually, this is stated only for $s\geq0$. To see this for $s<0$, simply note that $T_s h_k(x)=T_{\abs{s}}h_k(-x)=(-1)^k T_\abs{s}h_k(x)$.)
Consequently we have that
\[
\Phi(x) = \norm{f}_{L^2(\varphi)}^2
- 2\langle f,T_m f\rangle_{L^2(\varphi)} + C_\epsilon
= 2\sum_j u_j^2 - 2\sum_j u_j^2 m^j + C_\epsilon = \phi_f(m)
\]
as desired. 

Assumption A is immediate from \eqref{eq:GLM-pop}.
It remains to show that the pair satisfies Assumption B. 
To this end, recall that since $f'$ is of at most polynomial growth, we have $f\in H^{1}(\varphi)$,
where $H^1$ is the Sobolev space with norm  
\[
\norm{f}^2_{H^{1}(\varphi)}:=\int f(z)^2+f'(z)^{2}d\varphi(z)=\sum_{j\geq 0 }\left(1+j^{2}\right)u_{j}^{2}<\infty.
\]
We now turn checking \prettyref{eq:condition-1}-\prettyref{eq:condition-2}.
First note that for every $x$, 
\[
\nabla\Phi(x)=-2\sum j u_{j}^{2}m^{j-1}\cdot \nabla m.
\]
Thus, for every $x$, $\abs{\nabla \Phi(x)}\leq 2\sum ju_k^2 <2\norm{f}_{H^1}^2<\infty$,
where we used here that $\abs{m}\leq1$. Similarly for every $x$, $\nabla \Phi(x)\cdot e_1=\nabla\Phi(x)\cdot\nabla m= -2\sum ju_j^2 m^{k-1}$
so that $\abs{\nabla \Phi\cdot e_1}^2<4\norm{f}_{H^1}^4<\infty$.

Thus it suffices to check \prettyref{eq:condition-1}-\prettyref{eq:condition-2} for $\cL$ itself.
Here we have that if we let $\pi_x(v)=v-(v\cdot x) x$ be the projection on to the tangent space at $x$, we have
\[
\nabla\cL(x;a,\epsilon)=2(f(a\cdot x)+\eps-f(a_{1}))f'(a\cdot x)\pi_x a\,.
\]
Thus, by H\"older's inequality and the fact that $|x+y|^p \leq 2^p(|x|^p+|y|^p)$ for $p\geq 1$, for any $q\geq1$, if we take $\gamma>0$
and $r=\frac{2 q (q+\gamma)}{\gamma}$, 
we have that
\begin{align*}
\E\abs{\nabla\cL(x)}^{q} \leq 
2^{q}\Big(\E[(f(a\cdot x)-f(a_1))^{q+\gamma}]^{\frac{q}{q+\gamma}}+\E[\abs{\eps}^{q+\gamma}]^{\frac{q}{q+\gamma}}\Big)
\cdot \E[\abs{f'(a\cdot x)}^{r}]^{\frac{q}{r}} \cdot \E[\norm{a}_2^{r}]^{q/r}.
\end{align*}
The expectations involving $f,f'$ are bounded since $f$ and $f'$ are of at most
polynomial growth and $a\cdot x$ is Gaussian for every $x$; the expectation involving $\epsilon$ is bounded as long as $\eps$ has finite $p$-th moment for some $p>q$. The last term is bounded by~\eqref{eq:vector-norm-bound} since $a$ is a standard Gaussian vector. 
Taking $q=4+\iota$ yields \eqref{eq:condition-2} for $\iota$ small enough, after recalling our assumption 
that $\eps$ has finite $4+\delta$-th moment.

For \eqref{eq:condition-1}, observe that since $f$ and $f'$ are of at most polynomial
growth,
\begin{align*}
\E\abs{\nabla\cL\cdot e_{1}}^{2} & =\E\Big[\abs{f(a_{1}m+a_{2}\sqrt{1-m^{2}})-f(a_{1})+\eps}^{2}\abs{f'(a_{1}m+a_{2}\sqrt{1-m^{2}})}^{2}\abs{a_{1}}^{2}\Big]\leq C\,,
\end{align*}
again using H\"older's inequality together with the moment assumption on $\epsilon$.
\qed

\subsection{Proof of Proposition~\ref{prop:GLM-1-prop}}\label{app:GLM}
Let us first prove \eqref{eq:GLM-rep}. Since $\E[y\vert a] = f(a\cdot v_0)$,  we have
\[
\Phi(x) = \E [y (a\cdot x) - b(a\cdot x)] = \E[ f(a\cdot v_0) (a\cdot x)] - c,
\]
for some constant $c$ (in particular $c=\E b(a_1)$), where in the second equality, we have used the tower property.
Then as in the proof of \prettyref{prop:GLM-pop}, we see that upon taking $v_0 = e_1$ without loss of generality, 
\[
\E f(a\cdot v_0) a\cdot x = \E f(a_1) T_{m} h_1(a_1) = u_1(f) m\,.
\]
Finally, since $f$ is increasing, invertible, and differentiable, Gaussian integration-by-parts shows that 
$$u_1(f) = \E [Z f(Z)] = \E [f'(Z)]>0.$$
Assumption A is now evident from 
the representation formula \eqref{eq:GLM-rep}.

Let us now turn to proving that Assumption B holds.
To this end, note that as in \prettyref{prop:GLM-pop} the relevant inequalities hold 
for the population loss. Thus it suffices to show it for the true loss. 
Here we see that  if we let $\pi_x(v)=v-(v\cdot x) x$ be the projection on to the tangent space at $x$, we have
\[
\nabla_x \cL(y;a,x) = (y -b'(a\cdot x))\pi_x a.
\]
We then have by Cauchy-Schwarz
\[
\E \abs{\nabla_x \cL(y;a,x)\cdot e_1}^2 \leq C\sqrt{\E y^4+ \E b'(a\cdot x)^4}
\]
for some $C>0$. 
The first term under the radical is finite by assumption and the second term under is finite since  $b'=f$ is of at most exponential growth. 

For the gradient bound, note again that by the Cauchy-Schwarz inequality,
\[
\E \abs{\nabla_x \cL}^q =\E\Big[(y-b'(\tfrac{a\cdot x}{\sqrt N}))^q \norm{a}^q \Big]\leq  \E [(y-b'(a\cdot x))^{2q}] \cdot C(2q) N^{q/2}
\]
where $C(q)$ is as in \eqref{eq:vector-norm-bound}. Choosing $q = 4+\iota$  yields the desired bound by the same reasoning. \qed

\subsection{Proof of Proposition~\ref{prop:linear-regression-prop}}\label{app:linear-regression}

In the following $C$ will denote a constant that may change from line to line.
It is evident from \eqref{eq:lin-reg-phi} that $\Phi$ has information exponent $1$
and satisfies Assumption A.  It remains to check Assumption B.

Observe that since $\Phi(x)=-2m(x)+c$, where $m(x)=(v,x)$,
we have that $\nabla\Phi\cdot v=( \nabla\Phi,\nabla m) =\phi'(m)=-2$ and
that $\norm{\nabla\Phi}\leq C$. Thus it suffices to show the desired
bounds for $\nabla \cL$. To that end, notice that 
\begin{align*}
\nabla \cL(x;a,\epsilon)\cdot v & =\left(y-a\cdot x\right)\pi_{x}a\cdot v =(a\cdot(v-x)+\eps)(a\cdot\pi_x v)
\end{align*}
so that if we let $w=v-x$ and $\tilde{v}=\pi_x v$, then
\begin{align*}
\E[(\nabla \cL(x) \cdot v)^{2}] & =\E\left[\left(a\cdot w\right)^{2}\left(a\cdot\tilde{v}\right)^{2}\right]+\E[\epsilon^{2}]\E[(a\cdot\tilde{v})^{2}] \\
&\leq \sqrt{\E\abs{a\cdot w}^4\cdot\E\abs{a\cdot v}^4}+\E[\epsilon^{2}]\E[ \abs{a\cdot v}^2]\leq C
\end{align*}
where in the second line we used Cauchy-Schwarz and in the last inequality we used
our moment assumption to apply  \eqref{eq:moment-bound} and the fact that $\norm{\tilde{v}},\norm{w}\leq 2$.

For the norm bound, we take $4+\iota = 5$. Then we have
\begin{align*}
\E\norm{\nabla \cL}_{2}^{5}   \leq\E[\abs{a\cdot w+\epsilon}^{5}\cdot\norm{a}_{2}^{5}]
\leq 2^5 (\E[\abs{a\cdot w}^{5}\cdot\norm{a}_{2}^{5}]+\E[|\epsilon|^{5}]\cdot \E[\norm{a}_{2}^{5}])
\end{align*}
As $\norm{w}_{2}\leq2$, and $a$ is an i.i.d.\ centered random vector whose entries have finite $10$-th moments, by Cauchy--Schwarz and \eqref{eq:vector-norm-bound} and \eqref{eq:moment-bound}, we obtain
\[
\E\abs{a\cdot w}^{5}\cdot \E \norm{a}_{2}^{5}\leq\sqrt{\E\abs{a\cdot w}^{10}}\sqrt{\E\norm{a}_{2}^{10}}\leq C \norm{w}^{5}N^{5/2}=O(N^{5/2})
\]
Similary since $\epsilon$ has finite $5$-th moment, we see that
$\E|\epsilon|^{5}\cdot\E\norm{a}_{2}^{5}=O(N^{5/2}).$
Combining these bounds yields the desired.\qed

\subsection{Proof of \prettyref{prop:online-pca}}\label{app:streaming-pca}

Setting $m(x)=(v_0,x)$, we see from \eqref{eq:online-pca-phi} that
Assumption $A$ holds and the problem has information exponent 2. To verify the first part of Assumption B, observe that if we let $\tilde v_0=\pi_{x}v_0$, $$\nabla \cL(x)\cdot v_0= (Y\cdot x)(Y\cdot \tilde v_0)\,,$$
so that by Cauchy-Schwarz and \eqref{eq:moment-bound},
$\E(\nabla \cL(x)\cdot v_0)^{2}\leq C(\lambda).$
To verify the second part of Assumption B, $$\norm{\nabla \cL(x)}^{q}\leq\norm{YY^{T}x}^{q}\leq\norm{Y}^{q}\abs{(Y,x)}^{q}\,,$$ so that
by Cauchy-Schwarz, \eqref{eq:vector-norm-bound}, and \eqref{eq:moment-bound},
\[ 
\E\norm{\nabla \cL(x)}^{q}\leq\sqrt{\E\norm{Y}^{2q}}\sqrt{\E|(Y,x)|^{2q}}\leq C'N^{q/2}
\]
by assumption if we take $q=4+\iota = 5$. \qed

\subsection{Proof of Proposition~\ref{prop:tensor-pca-information-exponent}}\label{app:tensor-pca}
Taking the expectation of  \eqref{eq:tensor-pca-loss}, we have $\Phi(x)=-\lambda m(x)^{p}$,
where $m(x)=x\cdot v_{0}$ so that Assumption A holds, and the problem has information exponent $p$. For Assumption
B we argue as follows.

First note that $H(x)=(J,x^{\tensor p})=J(x,\ldots,x)$. If we let
$D$ denote the Euclidean derivative we have $\norm{\nabla H}\leq\norm{DH}$.
From this, it follows from writing out $DH$ that
\[
\E[\norm{DH(x)}^{q}]\leq C\E[\norm{J(x^{\tensor p-1},\cdot)}_{2}^{q}]\,,
\]
for some $C= C(p,q)$, where
\[
J(x^{\tensor p-1},\cdot)_{k}=\sum_{i_{1}\cdots i_{p-1}}J_{i_{1}\cdots i_{p-1}k}x_{i_{1}}\cdots x_{i_{p-1}},
\]
Let $I=(i_{1},\ldots,i_{p-1})$ denote a multi-index and $x_{I}=\prod_{i\in I}x_{i}$.
Observe that $J(x^{\tensor p-1},\cdot)$ is a centered i.i.d. vector
with

\begin{align*}
\E \Big[J\left(x^{\tensor p-1},\cdot\right)_{k}^{6}\Big] & =\E\Big[\sum_{I_{1},\ldots, I_{6}}J_{I_{1}1}\cdots J_{I_{6}1}x_{I_{1}}\cdots x_{I_{6}}\Big]\,,
\leq C
\end{align*}
for some $C>0$ depending on the law of $J$, where here we have
used that the entries of $J$ have finite $6$-th moment and that $x$ is a unit vector. Taking
$q=4+\iota=6$ we see
that since the entries of $J$ have finite $6$-th
moment, we have by  \eqref{eq:vector-norm-bound} that 
$\E\norm{DH}^{6}\leq CN^{3},$
for some $C>0$. This yields the second half of Assumption B. 

For the first half of Assumption B, note that if we let $\tilde v=\pi_{x}v$, 
then $\nabla H\cdot v=\sum_{i}\partial_{i}H(x)\tilde v_{i}$
is a sum of centered i.i.d.\ random variables with deterministic weights
$\tilde v$ with uniformly bounded $6$th moment, i.e.,
$\sup_{i}\E[\partial_{i}H(x)]^6\leq C$,
by the above argument. Thus by \eqref{eq:moment-bound} we have
$\E(\sum_{i}\partial_{i}H(x)\tilde v_{i})^{2}\leq C$
as desired since $\norm{\tilde v}\leq 1$.\qed

\subsection{Proof of Proposition~\ref{prop:mixture-model}}\label{app:mixture-model}

Observe that $Y\eqdist Z+\epsilon\mu$
with $Z\sim\cN(0,Id)$ and $\epsilon$ an independent Rademacher r.v. Writing $p = e^{h}/(e^{-h}+e^h)$ for some $h\in\R$ as above we have that   
$\Phi(x)=\phi(m(x))$ where
\[
\begin{aligned}
\phi(m)&=-\E\Big[\log\cosh\left(Z_{1}m+\sqrt{1-m^{2}}Z_{2}+\eps m(x)+h\right)\Big]\\
& = -\E\Big[\log\cosh\left(Z_1+\eps m+h\right)\Big].
\end{aligned}
\]
In the first line we used rotation invariance of the law of $Z$ and that $Z_1,Z_2$ are the first two entries of $Z$, and in the second line we used that $Z_1 m + Z_2 \sqrt{1-m^2} \eqdist Z_1$ since $m\in[-1,1]$.

Consequently
\[
\phi'(m) = -\E\tanh\left( Z_1+ \eps m +h\right) \eps.
\]
From this we see that
\[
\phi'(0) = -\E \big[\tanh(Z_1 + h)\eps\big]=
\begin{cases}
-\E \tanh(Z_1+h) \eps < 0 & p\neq 1/2\\
0 & p = 1/2.
\end{cases}	
\]
and 
\[
\phi''(m) = -\E\big[\mathrm{sech}^2 (Z_1+\eps m+h)\big]<0.
\]
Combining these two results yields: (a) that the information exponent is $1$ if $p\neq1/2$ and $2$ if $p=1/2$ and (b) that $\phi'(m)>0$ for $m>0$ the desired. \qed

\section{The discrete Bihari--LaSalle inequality}\label{app:bihari-lasalle}

For the purposes of completeness, in this appendix, we  provide a proof of the discrete version of the  Bihari--LaSalle inequality \eqref{eq:Bih-Lasalle}. 
Fix a $k>2$ and suppose that $a_{t}$ satisfies
\begin{align*}
a_{t} & =a+\sum_{\ell=0}^{t-1}c(a_{j})^{k}.
\end{align*}
for some $a,c>0$, then inductively, we can deduce that $m_{t}\geq a_{t}$ for all $t\geq 0$. To see
this, note that if we let 
\[
b_{t}=a+\sum_{j=0}^{t-1}c(m_{j})^{k-1}\,,
\]
it suffices to show that $b_{t}\geq a_{t}.$ Clearly $b_{0}=a_{0}$.
Suppose now that $b_{j}\geq a_{j}$; then
\[
b_{j+1}=a+\sum_{\ell=0}^{j}c (m_{\ell})^{k}=b_{j}+c(m_{j})^{k}\geq b_{j}+c(b_{j})^{k}\geq a_{j}+c(a_{j})^{k}=a_{j+1}
\]
where the first inequality follows by definition of $b_{j}$ and the
follows from the inductive hypothesis. Thus $m_{t}\geq b_{t}\geq a_{t}$.
It remains to lower bound $a_{t}$.

To this end, note that by definition, $a_{t}$ is non-decreasing for
$a>0$ and
\[
c=\frac{a_{t}-a_{t-1}}{a_{t-1}^{k}}\leq\int_{a_{t-1}}^{a_{t}}\frac{1}{x^{k}}dx=\frac{1}{(k-1)}\Big[\frac{1}{a_{t-1}^{k-1}}-\frac{1}{a_{t}^{k-1}}\Big].
\]
So by re-arrangement, 
\[
a_{t}\geq \Big(a_{t-1}^{-(k-1)}-(k-1)c\Big)^{-\frac{1}{k-1}}\quad\text{and\ensuremath{\quad}}a_{t}^{-(k-1)}\leq a_{t-1}^{-(k-1)}-(k-1)c.
\]
Since this holds for each $t,$ we see that 
\[
a_{t-1}^{-(k-1)}\leq a_{0}^{-(k-1)}-(k-1)c(t-1)
\]
from which it follows that 
\[
a_{t}\geq\frac{1}{(a_{0}^{-(k-1)}-(k-1)ct)^{\frac{1}{k-1}}}=\frac{a}{\left(1-(k-1)ca^{k-1}t\right)^{\frac{1}{k-1}}}.
\]
as desired. 



\bibliographystyle{abbrv}
\bibliography{spherical-OGD}

\end{document}